\documentclass{article}

\usepackage{graphicx}
\usepackage{multirow}
\usepackage[labelfont=bf,textfont=it,font=footnotesize]{caption}
\usepackage{parskip}
\usepackage{booktabs}
\usepackage{listings}
\usepackage{float}
\usepackage{subcaption}

\usepackage{amssymb}
\usepackage{amsthm}
\usepackage{amsmath}
\usepackage{nicefrac}        
\usepackage{siunitx}

\usepackage[algo2e,ruled,noend]{algorithm2e}
\usepackage[numbers,sort&compress]{natbib} 
\usepackage[utf8]{inputenc} 
\usepackage[T1]{fontenc}    
\usepackage{url}            
\usepackage{microtype}      
\usepackage{xcolor}         
\usepackage{enumitem}
\usepackage{tikz}
\usepackage{pgfplots}
\usepackage{pgfplotstable}
\usepackage{multirow}
\pgfplotsset{compat=1.17}

\usetikzlibrary{spy}

\usepackage{tabulary}
\usepackage{wrapfig}

\graphicspath{{Figures/}}
\DeclareGraphicsExtensions{.pdf,.png,.jpg}

\usepackage{hyperref}
\hypersetup{
    colorlinks,
    linkcolor={magenta},
    citecolor={blue},
    urlcolor={blue!80!black},
    breaklinks=true,
	plainpages=true
}
\usepackage{amsmath,amsfonts,bm, amsthm}
\usepackage{algorithm}
\usepackage[noend]{algorithmic}
\usepackage{nicefrac}        

\newcommand{\colref}[2]{\hyperref[#2]{#1~\ref*{#2}}}
\newcommand{\coloredref}[2]{\hyperref[#2]{#1~\ref*{#2}}}
\newcommand{\coloredsubref}[3]{\hyperref[#2]{#1~\ref*{#2}{#3}}}

\newcommand{\Figref}[1]{\colref{Figure}{#1}}

\newcommand{\eqnref}[1]{\colref{Equation}{#1}}

\newcommand{\Secref}[1]{\colref{Section}{#1}}

\def\eqref#1{\colref{Equation}{#1}}
\def\Eqref#1{\colref{Equation}{#1}}

\newcommand{\Algref}[1]{\colref{Algorithm}{#1}}

\newcommand{\Tabref}[1]{\colref{Table}{#1}}

\def\1{\bm{1}}

\def\rvx{{\mathbf{x}}}

\def\U{\vec{U}}

\def\rmJ{{\mathbf{J}}}

\def\rmV{{\mathbf{V}}}

\DeclareMathAlphabet{\mathsfit}{\encodingdefault}{\sfdefault}{m}{sl}
\SetMathAlphabet{\mathsfit}{bold}{\encodingdefault}{\sfdefault}{bx}{n}

\newcommand{\mathbbm}[1]{\text{\usefont{U}{bbm}{m}{n}#1}} 

\newcommand{\R}{\mathbb{R}}

\DeclareMathOperator*{\argmin}{arg\,min}

\theoremstyle{plain}

\theoremstyle{remark}

\theoremstyle{definition}

\theoremstyle{plain}

\theoremstyle{plain}

\theoremstyle{definition}
\newtheorem{definition}{Def.}[section]

\providecommand{\corollaryname}{Corollary}
\providecommand{\lemmaname}{Lemma}
\providecommand{\problemname}{Problem}
\providecommand{\remarkname}{Remark}
\providecommand{\theoremname}{Theorem}
\newtheorem{theorem}{Theorem}

\newcommand{\mvec}[1]{{\underline{#1}}}

\newcommand{\bunderline}[2][4]{\underline{#2\mkern-#1mu}\mkern#1mu}

\newcommand{\grad}{\bunderline{{\nabla}}}

\renewcommand\vec[1]{\ensuremath\mathbf{#1}}

\usepackage{arxiv}

\title{
Differentiable Spline Approximations
}

\author{
Minsu Cho$^{2\dagger}$, Aditya Balu$^{1\dagger}$, Ameya Joshi$^2$, Anjana Deva Prasad$^1$, Biswajit Khara$^1$, Soumik Sarkar$^1$, \\ Baskar Ganapathysubramanian$^1$, Adarsh Krishnamurthy$^1$, Chinmay Hegde$^2$~{\thanks{$^1$The author is with Iowa State University, $^2$The author is with New York University. $^\dagger$Equal contribution.}}
}

\begin{document}

\maketitle




\begin{abstract}
The paradigm of differentiable programming has significantly enhanced the scope of machine learning via the judicious use of gradient-based optimization. However, standard differentiable programming methods (such as autodiff) typically require that the machine learning models be differentiable, limiting their applicability. Our goal in this paper is to use a new, principled approach to extend gradient-based optimization to functions well modeled by splines, which encompass a large family of piecewise polynomial models. We derive the form of the (weak) Jacobian of such functions and show that it exhibits a block-sparse structure that can be computed implicitly and efficiently. Overall, we show that leveraging this redesigned Jacobian in the form of a differentiable ``layer'' in predictive models leads to improved performance in diverse applications such as image segmentation, 3D point cloud reconstruction, and finite element analysis.
\end{abstract}

\subsection*{Keywords}

Differentiable NURBS Layer $|$
NURBS $|$
Geometric Deep Learning $|$
Surface Modeling

\section{Introduction}\label{Sec:Introduction}

\noindent\textbf{Motivation:} Differentiable programming has been a paradigm shift in algorithm design. The main idea is to leverage gradient-based optimization to optimize the parameters of the algorithm, allowing for end-to-end trainable systems (such as deep neural networks) to exploit structure in data and achieve better performance. This approach has found use in a large variety of applications such as scientific computing~\citep{Innes2020ALGORITHMICD, Innes2019ADP, Schafer2020ADP}, image processing~\citep{,li2018differentiableprog}, physics engines~\citep{Degrave2017ADP}, computational simulations~\citep{alnaes2015fenics}, and graphics~\citep{li2018differentiable, Chen2019LearningTP}. One way to leverage differentiable programming modules is to encode additional {structural priors} as ``layers'' in a larger machine learning model. Inherent structural constraints such as monotonicity, or piecewise constancy, are particularly prevalent in applications such as physics simulations, graphics rendering, and network engineering. In such applications, it may be beneficial to build models that obey such priors {by design}.

\noindent\textbf{Challenges:} For differentiable programming to work, all layers within the model must admit simple gradient calculations; however, this poses a major limitation in many settings. For example, consider computer graphics applications for rendering 3D objects~\citep{kindlmann2003curvature,gross1995new,loop2006real}. A common primitive in such cases is a \emph{spline} (or a piecewise polynomial) function which either exactly or approximately interpolates between a discrete set of points to produce a continuous shape or surface. Similar spline (or other piecewise polynomial) approximations arise in partial differential equation (PDE) solvers~\citep{hughes2005isogeometric}, network flow problems~\citep{balakrishnan1989composite}, and other applications. 

For such problems, we would like to compute gradients ``through'' operations involving spline approximation. However, algorithms for spline approximation often involve discontinuous (or even discrete) co-domains and may introduce undefined (or even zero) gradients. Generally, embedding such functions as layers in a differentiable program, and running automatic differentiation on this program, requires special care. A popular solution is to relax these non-differentiable, discrete components into continuous approximations for which gradients exist. This has led to recent advances in differentiable sorting~\citep{Blondel2020FastDS, Cuturi2019DifferentiableRA}, dynamic programming~\citep{Mensch2018DifferentiableDP}, and optimization~\citep{djolonga2017differentiable,cvxpylayers2019,Deng_2020}. 


\noindent\textbf{Our contributions:} We propose a principled approach for differentiable programming for spline functions \emph{without} the use of continuous relaxation\footnote{While tricks such as straight-through gradient estimation~\citep{Bengio2013EstimatingOP} also avoid continuous relaxation, they are heuristic in nature and may be inaccurate for specific problem instances~\citep{yin2019understanding}.}. For the forward pass, we leverage fast algorithms for computing the optimal projection of any given input onto the space of piecewise polynomial functions. For the backward pass, we leverage a fundamental \emph{locality} property in splines that every piece (or basis function) in the output approximation only interacts with a few other elements. Using this, we derive a weak form of the Jacobian for the spline operation, and show that it exhibits a particular block-structured form. While we focus on spline approximation in this paper, our approach can be generalized to any algorithmic module with piecewise smooth outputs. Our specific contributions are as follows:

\begin{enumerate}[leftmargin=*, noitemsep, parsep=0pt]
    \item We propose the use of spline function approximations as ``layers'' in differentiable programs.
    \item We derive efficient (nearly-linear time) methods for computing forward and backward passes for various spline approximation problems, showing that the (weak) Jacobian in each case can be represented using a \emph{block sparse} matrix that can be efficiently used for backpropagation.  
    \item We show applications of our approach in three stylized applications: image segmentation, 3D point cloud reconstruction, and finite element analysis for the solution of partial differential equations.
\end{enumerate}

\noindent\textbf{Related Work} Before proceeding, we briefly review related work.


\noindent\textbf{Extensions of autodiff:} Automatic differentiation (autodiff) algorithms enable gradient computations over basic algorithmic primitives such as loops, recursion, and branch conditions~\citep{baydin2018autodiff}. However, introducing more complex non-differentiable components requires careful treatment due to undefined or badly behaved gradients. For example, in the case of sorting and ranking operators, it can be shown that the corresponding gradients are either uninformative or downright pathological, and it is imperative the operators obey a `soft' differentiable form. \citet{Cuturi2019DifferentiableRA} propose a differentiable \emph{proxy} for sorting based on optimal transport. \citet{Blondel2020FastDS} improve this by proposing a more efficient differentiable sorting/ranking operator by appealing to isotonic regression. \citet{Berthet2020LearningWD} introduce the use of stochastic perturbations to construct smooth approximations to discrete functions, and other researchers have used similar approaches to implement end-to-end trainable top-$k$ ranking systems~\citep{Xie2020DifferentiableTO,Lee2020ADR}. Several approaches for enabling autodiff in optimization have also been researched~\citep{Pogan2020Differentiation,amos2019optnet, cvxpylayers2019,Mensch2018DifferentiableDP}.


\noindent\textbf{Structured priors as neural ``layers'':} As mentioned above, one motivation for our approach arises from the need for enforcing structural priors for scientific computing applications. Encoding non-differentiable priors such as the solutions to specific partial differential equations~\citep{Sheriffdeen2019AcceleratingPI}, geometrical constraints~\citep{Joshi2020InvNetEG, Chen2019LearningTP}, and spatial consistency measures~\citep{djolonga2017differentiable} perform well but typically require massive amounts of structured training examples.  

\noindent\textbf{Spline approximation:} 
Non-Uniform Rational B-splines (NURBS) are commonly used for defining spline surfaces for geometric modeling~\citep{10.5555/265261}. NURBS surfaces offer a high level of control and versatility; they can also compactly represent the surface geometry. The versatility of NURBS surfaces enables them to represent more complex shapes than B\`ezier or B-splines. Several frameworks that leverage deep learning are beginning to use NURBS representations. \citet{minto2018deep} use NURBS surfaces fitted over the 3D geometry as an input representation for the object classification task of ModelNet10 and ModelNet40 datasets. \citet{erwinski2016neural} presented a neural-network-based contour error prediction method for NURBS paths. \citet{fey2018splinecnn} present a new convolution operator based on B-splines for irregular structured and geometric input, e.g., graphs or meshes. 
Very recently, \citep{sharma2020parsenet} perform point cloud reconstruction to predict a B-spline surface, which is later processed to obtain a complete CAD model with other primitives ``stitched'' together.

\noindent\textbf{Differentiable PDE solvers:} With the advent of deep learning, there has been a recent rise in the development of differentiable programming libraries for physics simulations~\citep{hu2019difftaichi,qiao2020scalable}. Most often, the physics phenomena are represented using partial differential equations (PDEs)~\citep{sanchez2020learning,holl2020learning}. Considerable effort has gone into designing physics-informed loss functions~\citep{raissi2017physics1,raissi2018hidden,kharazmi2021hp} whose optimization leads to desired solutions for PDEs. Due to space limitations, we defer to a detailed survey of this (vast) area by~\citet{cai2021physics}.

\section{Differentiable Spline Approximation}
\label{sec:diffpiece}

We now introduce our framework, Differentiable Spline Approximation (DSA), as an approach to estimate gradients over piecewise polynomial operations. Our main goal will be to estimate easy-to-compute forms of the (weak) Jacobian for several spline approximation problems, enabling their use within backward passes in general differentiable programs. 


\textbf{Setup.} We begin with some basic definitions and notation. Let $f \in \R^n$ be a vector where the $i{\textsuperscript{th}}$ element is denoted as $f_i$. 
Let us use $[n] = \{1,2,\ldots,n\}$ to denote the set of all coordinate indices. For a vector $f \in \R^n$ and an index set $I \subseteq [n]$, let $f_I$ be the restriction of $f$ to $I$, i.e., for $i\in I$, we have $f_I(i):=f_i$, and $f_I(i):=0$ for $i \notin I$. Now, consider any fixed partition of $[n]$ into a set of disjoint intervals $\mathcal{I}=\{I_1, \ldots, I_k\}$ where the number of intervals $|\mathcal{I}|=k$. The $\ell_2$-norm of $f$ is written as $\|f\|_2:=\sqrt{\sum_{i=1}^n f_i^2}$ while the $\ell_2$ distance between $f, g$ is written as $\|f-g\|_2$.


We first define the notion of a \emph{discretized} $k$-\emph{spline}. Note that the use of ``spline" here is non-standard and somewhat more general than what is typically encountered in the literature. (Indeed, the spline concept used in computer graphics is a special instance of this definition; we explain further below.) 

\begin{definition}[Discretized $k$-spline]
\label{def:ppf}
A vector $h \in \R^n$ is called a discretized $k$-spline with degree $d$ if: (i) there exists a partition of $[n]$ into $k$ disjoint intervals $I_1, \ldots, I_k$; (ii) within each interval $I_i$, the coefficients of $h_j$, $j \in I_i$, can be {perfectly} interpolated by some polynomial function of degree $d$. 
\end{definition}


Let us illustrate this by an example. Suppose that $d=1$ and $k=5$. Then, $h$ is a discretized $k$-spline with degree $d$ if, in a ``line plot'' of the vector $h$ (i.e., we interpolate the 2D points $(j,h_j)$ for all $j \in [n]$), we see up to $k=5$ distinct linear pieces. 
A different way to interpret this definition is that we start with a piecewise degree-$d$ polynomial function $H : \R \rightarrow \R$ with $k=5$ pieces (with suitably defined knot points, which are the location of the intervals $I$), and evaluate $H$ at any $n$ equally spaced points in its domain. This gives us a vector $h \in \R^n$, which we call a discretized $k$-spline. In contrast with traditional splines, we allow $H$ to be arbitrarily defined \emph{at} the knot points and require no specific continuity or differentiability properties. Therefore, our definition encompasses all standard spline families (including interpolating/approximating splines such as smoothing-, cubic-, and B-splines).


\subsection{Spline Approximation}

Our focus in this paper is the problem of computing the best possible spline fit to a given set of data points (where both the parameters of the spline as well as the knot vectors are allowed to be variable). 


We provide an algebraic interpretation of this problem. For a given vector space $\R^n$, consider $S^k_d$, the set of all discretized $k$-splines with degree $d$. Since (standard) splines are vector spaces for a fixed set of knots, one can easily see that for any fixed partition of $[n]$ into $k$ subsets, the family of discretized $k$-splines is a $k(d+1)$-dimensional subspace of $\R^n$. Now suppose that the knot indices are allowed to vary. The number of possible partitions is finite (of the order of $\binom{n}{k}$), and therefore the set $S^k_d$ is a \emph{finite union of subspaces}, or a nonlinear submanifold, embedded in $\R^n$. 

Therefore, the problem of discretized $k$-spline approximation can be viewed as an \emph{orthogonal projection} onto this nonlinear manifold. Consider any arbitrary vector $x \in \R^n$ (we can think of $(i,x_i)$ as a set of $n$ data points to which we are trying to fit a $k$-spline). Then, the best $k$-spline fit to $x$ (in the sense of $\ell_2$ distance) amounts to solving the optimization problem:
\begin{align}
    {F(x)} = \argmin_{h} \frac{1}{2} \|x- h\|_2^2 = \frac{1}{2}\sum_{i=1}^n (x_i - h_i)^2 ~~
    \text{s.t.~} h \in S^k_d
\end{align}

This operation resembles standard spline regression. But it is strictly more general since this requires not only optimizing piecewise spline parameters but \emph{also the knot indices}. Crucially, we note that $F$ is both a non-differentiable and a non-convex map. Nevertheless, such an orthogonal projection can be computed in polynomial (in fact, nearly-linear)  time~\citep{DBLP:conf/vldb/JagadishKMPSS98,acharya2015fast} using many different techniques, including dynamic programming. This forms the \emph{forward pass} of our DSA ``layer''. 



Our first main conceptual contribution is a formal derivation of the \emph{backward pass} of the orthogonal projection operation. Strictly speaking, the Jacobian is not well-defined due to the non-differentiable nature of the forward pass (owing to the non-differentiability built into the definition of the $k$-spline). Therefore, we will instead be deriving the so-called ``weak'' form of the Jacobian (borrowing terminology from~\citet{Blondel2020FastDS}).

We leverage two properties of the projection operation: (1) the output of the forward pass $h$ corresponds to a partition of $[n]$, that is, each element of $h_j$ corresponds to a \emph{single} interval, $I_j$, and (2) within each interval, the least-squares operation is continuous and differentiable. 
The first property ensures that every element $x_i$ contributes to only a single piece in the output $h$.
Given that the sub-functions from the piecewise partitioning function are smooth, we also observe that the size of each block corresponds to the size of the partition, $I_i$. Using this observation, we get:

\begin{theorem}
    \label{thm:blockwise}
    The Jacobian of the operation $F$ with respect to $x \in \R^n$ can be expressed as a \emph{block diagonal} matrix, $\rmJ \in \R^{n \times n}$, whose $(s,t)^{\textrm{th}}$ entry obeys:
    \begin{eqnarray}
        \rmJ_x(F(x)) (s, t) = \frac{\partial h(x)_s}{\partial x_t} = \begin{cases}
        \frac{\partial h_{I_i}(x)_s}{\partial x_t} &\textnormal{if } s, t \in I_i \\
        0 &\textnormal{otherwise}
    \end{cases}
    \end{eqnarray}
\end{theorem}

As a concrete instantiation of this result, consider the case $d=0$. This is the case where we wish to best approximate the entries of $x$ with at most $k$ ``horizontal'' pieces, where the break-points are obtained during the forward pass\footnote{In the data summarization literature, this class of functions is sometimes called $k$-histograms~\citep{DBLP:conf/vldb/JagadishKMPSS98}}. Call this approximation $h$. Then, the Jacobian of $h$ with respect to $x$ forms the block-diagonal matrix $\mathbf{J} \in \mathbb{R}^{n \times n}$:
\begin{align}
    \mathbf{J} =
    \begin{bmatrix}
        \mathbf{J}_1 & \mathbf{0} & \ldots & \mathbf{0} \\
        \mathbf{0} & \mathbf{J}_2 & \ldots & \mathbf{0} \\
        \vdots & \vdots & \ddots & \vdots \\
        \mathbf{0} & \mathbf{0} & \ldots & \rmJ_k
    \end{bmatrix}
    \label{Eq:JMatrix}
\end{align}
where all entries of each block, $\rmJ_i \in \mathbb{R}^{|I_i|\times|I_i|}$ are \emph{constant} and equal to $1/|I_i|$, i.e., they are row/column-stochastic. Note that the sparse structure of the Jacobian allows for fast computation and that computing the Jacobian vector product $\rmJ^T \nu$ for any input $\nu$ requires $O(n)$ running time. As an additional benefit, the decoupling induced by the partition enables further speed up in computation via parallelization. See the Appendix for proofs, as well as derivations of similar Jacobians for $k$-spline approximation of any degree $d \geq 1$, and generalization to 2D domains (surface approximation). In \Secref{sec:expts} we demonstrate the utility of this approach for a 2D segmentation (i.e., piecewise constant approximation) problem, similar to the setting studied in~\cite{djolonga2017differentiable}.

\subsection{Differentiable NURBS}\label{SubSec:NURBS}
We now switch to a slightly different setting involving a special spline family known as non-uniform rational B-splines (NURBS), which are common in geometric modeling. Mathematically, a NURBS curve is a continuous function $\vec{C} : \R \rightarrow \R$ defined as follows. Construct any knot vector $u$ (i.e. a non-decreasing sequence of real coordinate values) and fix degree $d$. Recursively define a sequence of basis functions, $N_i^d : \R \rightarrow \R$ computed using the \emph{Cox-de Boor formula}: 
\begin{equation}
	N_i^d(u)=\frac{u-u_i}{u_{i+d}-u_i}N_i^{d-1}(u)+\frac{u_{i+d+1}-u}{u_{i+d+1}-u_{i+1}}N_{i+1}^{d-1}(u),~
	N_i^0(u) = \left\{
	\begin{array}{l l}
		1  \quad \mbox{if $u_i\leq u\leq u_{i+1}$} \\
		0  \quad \mbox{otherwise}
	\end{array} \right.
	\label{eq:NURBSBasis}
\end{equation}
for $d=1,2,\ldots$. In the uniform case (where the knots are equally spaced), each $N_i^d$ can be viewed as being generated by recursively convolving a box function with $N_i^{d-1}$. The non-uniform case cannot be written as a convolution, but the intuition is similar. With these basis functions in hand, the NURBS curve $\vec{C}$ is defined as the rational function:
\begin{equation}
	\vec{C}(u)=\frac{\sum_{i=0}^n{N_i^d(u) w_{i}\vec{P}_{i}}}{\sum_{i=0}^n{N_i^d(u) w_{i}}} ,
	\label{eq:NURBS}
\end{equation}
where $P_i,~i=0,1,\ldots,t$ are called \emph{control points} and $w_i$ are corresponding non-negative weights. The number of control points is related to the number of knots $k$ and curve degree $d$ as follows: $k = t + d + 1$. For simplicity, assume that all weights are equal to one. The basis functions in NURBS add up to one uniformly for each $u$ (this is called the \emph{partition of unity} property). Therefore: 
\begin{equation}
\vec{C}(u)= \sum_{i=0}^t{{N_i^d(u)\vec{P}_{i}}} ,
	\label{eq:NUBS}
\end{equation}
In summary, the NURBS curve is parametrically defined via the control points and the knot positions. This discussion is for 1D curves, but extension to higher-order surfaces is conceptually similar.

Consider implementing NURBS as a differentiable ``layer'' where the inputs are the knot positions and control points. The forward pass through this layer simply consists of evaluating \Eqref{eq:NUBS} via the recursive~\Eqref{eq:NURBSBasis}, and storing the various basis functions (and their spans) for further use.

However, the backward pass is a bit more tricky, once again due to the \emph{non-differentiable} nature of $\vec{C}$. The gradient with respect to the control point coordinates, $\vec{P}$ is immediate (since the mapping from $\vec{P}$ to $\vec{C}$ is linear). However, the gradient with respect to the \emph{knot} positions, $u_i$, is not well-defined due to the non-differentiable nature of the \emph{base cases} of the recursion (which are box functions specified in terms of $u_i$). Once again, we see that the non-differentiability of NURBS is built into its very definition, and this affects numerics.


To resolve this, we propose the following approach to compute an (approximate) Jacobian of $\vec{C}$. The main source of the issue is the derivative of the box-car function $N_i^0(u) = \mathbf{1}_{[u_i, u_{i+1})}$ with respect to the knot points, which is not well defined. However, $N_i^0(u)$ can be viewed as the difference between convolutions of the unit step function with $\delta_{u_i}$ and $\delta_{u_{i+1}}$, where $\delta$ is the Dirac delta defined over the real line. We smoothly approximate the delta function by a Gaussian function with small enough bandwidth hyperparameter $\sigma$: $\delta(u_i) \approx g(u) = \exp(-(u - u_i)/2\sigma^2)$. This function is now differentiable with respect to $u_i$, with $g'(u) = \frac{u-u_i}{\sigma^2} g(u)$. Convolutions and differences are linear, and hence the derivative is the basis function times a multiplicative factor. Finally, a similar approach as the Cox-de Boor recursion (\Eqref{eq:NURBSBasis}) can be used to reconstruct the derivatives for all basis functions of higher order. See \Algref{Alg:Backward} for pseudocode and the Appendix for details. 

 
Let us probe the structure of this Jacobian a bit further. Suppose we evaluate the curve $\vec{C}$ at $n$ arbitrary domain points. There are slightly less than $k$ control points, and therefore the Jacobian is roughly of size $n \times O(k)$. 
However, due to the recursive nature of the definition of basis functions, the \emph{span} (or support) of each basis function is small and only touches $d+1$ knots; for example, only 2 knots affect $N^0_i$, only 3 knots impact $N^1_i$, and so on. This endows a natural sparse structure on the Jacobian. Moreover, for a fixed order parameter $d+1$, the span is constant~\citep{10.5555/265261}; therefore, assuming evenly spaced evaluation points, we have the same number of nonzeros. Therefore, the Jacobian exhibits an interesting \emph{Toeplitz} structure (unlike the block diagonal matrix in the case of \eqnref{Eq:JMatrix}), thereby enabling efficient evaluation during any gradient calculations. We show below in \Secref{sec:expts} that automatic differentiation using this approach surpasses existing NURBS baselines.

\begin{algorithm}[t!]
    \caption{Backward pass for NURBS Jacobian (for one curve point , $\vec{C}(u)$)
    \label{Alg:Backward}}
    \newcommand\mycommfont[1]{\footnotesize\ttfamily\textcolor{blue}{#1}}
    \SetCommentSty{mycommfont}
    \SetNoFillComment


    \text{$\vec{P'}$, $\vec{U'}$: gradients of $\vec{C}$ w.r.t.\ $\vec{P}$, $\vec{U}$}
    
    \text{Initialize: $\vec{P'}, \vec{U'} \rightarrow 0$}\\
            \text{Retrieve $u_{span}$, $N_i^d$, $\vec{C}(u)$ calculated during forward pass}

            \tcc{$u_{span}$ is the index of knot position}
            \tcc{$N_{i}^d$ is the basis function of degree $d$}
            \tcc{$\vec{C}(u)$ is the evaluated curve point}
            \For{$h=0:d+1$}{
             $\vec{P'}_{u_{span}+h} = N_h^d$ \tcp{easy since $\vec{C}$ is a linear function of $\vec{P}$.}
             $\vec{U'}_{u_{span}+h} = N_h^d \, \vec{U}_{u_{span}+h}$ \tcp{due to Gaussian approximation; see discussion below.}
            }
\end{algorithm}

\subsection{Differentiable Finite Element PDE Solvers}
\label{sec:fem}

Next, we see how spline approximations can be used to improve finite element analysis for solving PDEs. Popular recent efforts for solving PDEs using autodiff construct ``physics-informed'' solvers~\citep{raissi2017physics1,raissi2018hidden}, while other efforts have been made to utilize variational~\citep{kharazmi2021hp} or adjoint-based derivative methods~\citep{holl2020learning}. However, these approaches come with challenges while used in conjunction with autodiff packages, and gradient pathologies pose a major barrier~\citep{wang2020understanding}. 

Using our principles developed above, we propose an alternative PDE solution approach via \emph{differentiable finite elements}. PDE solvers based on Finite Element Methods (FEM) are ubiquitous, and we provide a very brief primer here. Consider a domain $\Omega$ and a differential system of equations:
	\begin{align}
	\mathcal{N}[\U(\mvec{u})] &=F(\mvec{u}),\quad \mvec{u}\in \Omega, \label{pde:abstract-equation} 
	\end{align}
where $\mathcal{N}$ denotes the differential operator and $\U : \Omega \rightarrow \R$ is a continuous field variable; it is common to specify additional boundary constraints on $\U$. The \emph{Galerkin method} converts solving for the best possible $\U$ (which is a continuous variable) into a discrete problem by first looking at the weak form:  
$	R(\U) = \int_{\Omega} V \left[ \mathcal{N}(\U) -F \right] d\mvec{u},$
where $V$ is called a \emph{test function} (and the weak form may involve some integration by parts), and rewriting this weak form in terms of a finite set of {basis} coefficients. A typical set of basis functions $\Phi_j$ is obtained by (piecewise) concatenation of polynomials, each defined over {elements} of a given partition of $\Omega$ (also called a \emph{mesh}). Commonly used choices include Lagrange polynomials, defined by:
\begin{align}
    p_{i,d}^{r}(\mvec{u}) = \sum_{r=1}^d \U_r \prod_{\substack{0\leq m \leq d\\ m\not = r}} \frac{\mvec{u} - u_m}{u_r - u_m} \; \text{s.t.} \; x_r \in [-1,1] 
\end{align}
where $\{u_0, u_1,\ldots,u_d\}$ are a finite set of nodes (akin to control points in our above discussion, except in this case the splines interpolate the control points) and $\U_r$ is the corresponding coefficient.
We use this collection of basis functions $\Phi_j$ to represent $\U$:
\begin{align}\label{def:fem-function-approximation}
	\U(\mvec{u}) = \sum_{j = 1}^{\text{\#{nodes}}} \Phi_j(\mvec{u}) \U^d_j
\end{align}
and likewise for $V$. (The resemblance with~\Eqref{eq:NURBS} above should be clear, and indeed NURBS basis functions could be an alternative choice.) Plugging the discrete coefficient representation $\U^c := \{\U^c_j\}$ into the definition of $R$, we get a standard Finite Element form, 
\begin{align}
	R(\U^c, V^c) = B(\U^c,V^c) - L(V^c)
\end{align}
where $B(\U^c, V^c)$ is the discrete form (bilinear for linear operators) that encodes the differential operator and $L(v)$ is a linear functional involving the forcing function. For most PDE operators (including linear elliptic operators), one can form the \emph{energy functional} by using $U$ as the test function: 
\begin{align}\label{eq:loss}
	J(\U^c) = \frac{1}{2} B(\U^c,\U^c) - L(\U^c).
\end{align}
Optimization of this energy functional can now be performed using gradient-based iterations evaluated by automatic differentiation. This is powerful since formal techniques exist (e.g., Galerkin Least Squares~\cite{bochev2009least}) that reformulate weak forms of PDEs into equivalent energy functionals. The key aspect to note here is that differentiating ``through" the differential operator $\mathcal{N}$ (embedded within $B$) requires derivative computations of the piecewise polynomial basis functions $\Phi_j$s, and therefore our techniques developed above are applicable. 

\section{Experiments}
\label{sec:expts}

We have implemented the DSA framework (and its different applications provided below) by extending \texttt{autograd} functions in Pytorch. We also provide the capability to run the code using CUDA for GPU support. All the experiments were performed using a local cluster with 6 compute nodes and each node having 2 GPUs (Tesla V100s with 32GB GPU memory). All training were performed using a single GPU. Each experiment shown below is performed multiple times with different random seeds, and the average value with error bars is provided. Due to limited space, we provide three interesting applications of spline approximations here (see Appendix for additional examples).

\noindent\textbf{Image segmentation:}\label{subsec:segmentation} We begin with implementing a 2D piecewise constant splines regression approach for the image segmentation problem using a UNet~\citep{ronneberger2015u}. For differentiation, we use the formulation of splines discussed in \Secref{sec:diffpiece}.
We analyze the efficacy of our approach by adding a piecewise constant DSA layer as the final layer of our network ($M_\text{DSA}$). We compare this approach with the baseline model without the piecewise constant layer ($M_\text{baseline}$).


We train two models ($M_\text{DSA}$, $M_\text{baseline}$) on two different segmentation tasks: the Weizmann horse dataset~\citep{borenstein2004learning} and the Broad Bioimage Benchmark Collection dataset~\citep{ljosa2012annotated} (publicly available under Creative Commons License). We split both the Weizmann horse dataset and Broad Bioimage Benchmark Collection dataset into train and test with 85\% and 15\% of the dataset. We use binary cross-entropy error between the ground truth and the predicted segmentation map. We use the same architecture and hyper-parameters for both models (see Appendix for details.)


\begin{figure}[t!]
    \centering
    \setlength{\tabcolsep}{1pt}
    \renewcommand{\arraystretch}{0.2}
    \def\sw{0.11\linewidth}
    \centering
    \begin{tabular}{c c c c c c c c}
        Original & Ground & $M_\text{baseline}$ & $M_\text{DSA}$ &
        Original & Ground & $M_\text{baseline}$ & $M_\text{DSA}$ \\        
        \raisebox{4pt}{\centering \includegraphics[width=\sw]{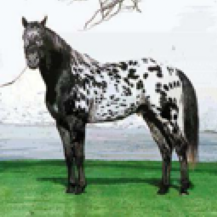}} &
        \raisebox{4pt}{\centering \includegraphics[width=\sw]{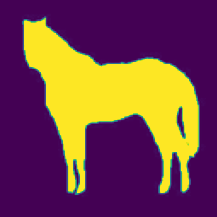}} & 
        \begin{tikzpicture}[spy using outlines={circle,yellow,magnification=3,size=0.8cm, connect spies}]
        \node {\includegraphics[width=\sw]{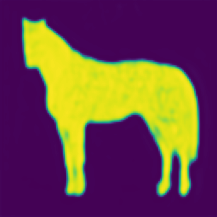}};
        \end{tikzpicture} &
        \begin{tikzpicture}[spy using outlines={circle,yellow,magnification=3,size=0.8cm, connect spies}]
        \node {\includegraphics[width=\sw]{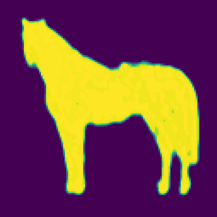}};
        \end{tikzpicture} &
        \raisebox{4pt}{\centering \includegraphics[width=\sw]{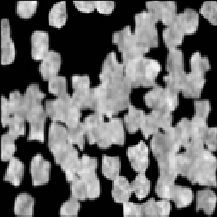}} &
        \raisebox{4pt}{\centering \includegraphics[width=\sw]{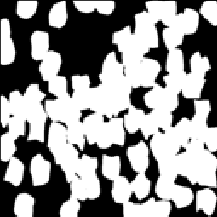}} & 
        \begin{tikzpicture}[spy using outlines={circle,yellow,magnification=3,size=0.8cm, connect spies}]
        \node {\includegraphics[width=\sw]{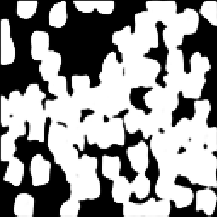}};
        \spy on (0,0.3) in node [left] at (0.7,-0.3);
        \end{tikzpicture} &
        \begin{tikzpicture}[spy using outlines={circle,yellow,magnification=3,size=0.8cm, connect spies}]
        \node {\includegraphics[width=\sw]{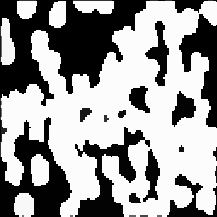}};
        \spy on (0,0.3) in node [left] at (0.7,-0.3);
        \end{tikzpicture} \\
        \raisebox{4pt}{\centering \includegraphics[width=\sw]{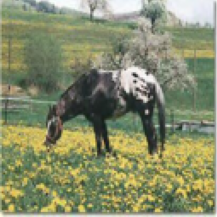}} &
        \raisebox{4pt}{\centering \includegraphics[width=\sw]{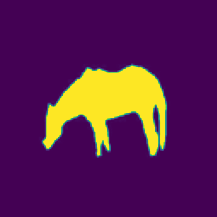}} & 
        \begin{tikzpicture}[spy using outlines={circle,yellow,magnification=3,size=0.8cm, connect spies}]
        \node {\includegraphics[width=\sw]{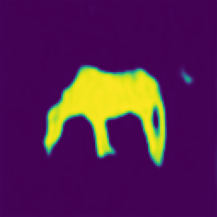}};
        \end{tikzpicture} &
        \begin{tikzpicture}[spy using outlines={circle,yellow,magnification=3,size=0.8cm, connect spies}]
        \node {\includegraphics[width=\sw]{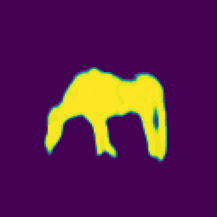}};
        \end{tikzpicture} &
        \raisebox{4pt}{\centering \includegraphics[width=\sw]{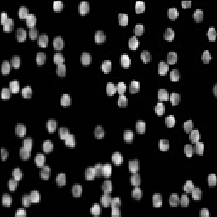}} &
        \raisebox{4pt}{\centering \includegraphics[width=\sw]{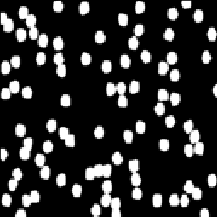}} & 
        \begin{tikzpicture}[spy using outlines={circle,yellow,magnification=3,size=0.8cm, connect spies}]
        \node {\includegraphics[width=\sw]{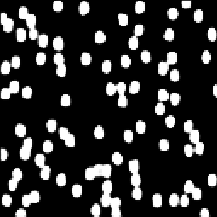}};
        \spy on (-0.1,0.3) in node [left] at (0.7,-0.3);
        \end{tikzpicture} &
        \begin{tikzpicture}[spy using outlines={circle,yellow,magnification=3,size=0.8cm, connect spies}]
        \node {\includegraphics[width=\sw]{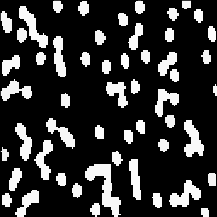}};
        \spy on (-0.1,0.3) in node [left] at (0.7,-0.3);
        \end{tikzpicture} \\
        \raisebox{4pt}{\centering \includegraphics[width=\sw]{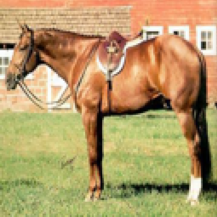}} &
        \raisebox{4pt}{\centering \includegraphics[width=\sw]{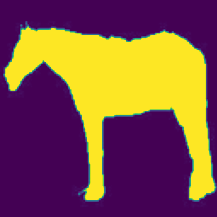}} & 
        \begin{tikzpicture}[spy using outlines={circle,yellow,magnification=3,size=0.8cm, connect spies}]
        \node {\includegraphics[width=\sw]{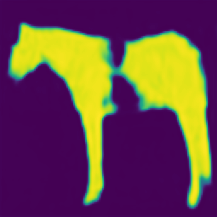}};
        \end{tikzpicture} &
        \begin{tikzpicture}[spy using outlines={circle,yellow,magnification=3,size=0.8cm, connect spies}]
        \node {\includegraphics[width=\sw]{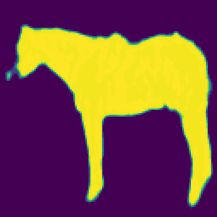}};
        \end{tikzpicture} &
        \raisebox{4pt}{\centering \includegraphics[width=\sw]{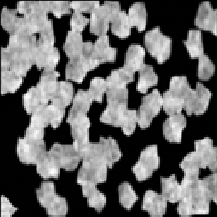}} &
        \raisebox{4pt}{\centering \includegraphics[width=\sw]{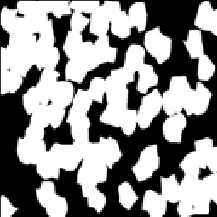}} & 
        \begin{tikzpicture}[spy using outlines={circle,yellow,magnification=3,size=0.8cm, connect spies}]
        \node {\includegraphics[width=\sw]{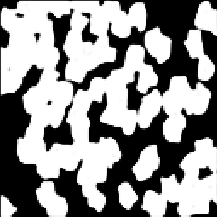}};
        \spy on (0,0.3) in node [left] at (0.7,-0.3);
        \end{tikzpicture} &
        \begin{tikzpicture}[spy using outlines={circle,yellow,magnification=3,size=0.8cm, connect spies}]
        \node {\includegraphics[width=\sw]{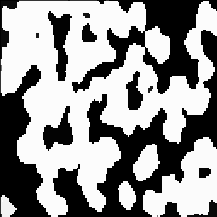}};
        \spy on (0,0.3) in node [left] at (0.7,-0.3);
        \end{tikzpicture} \\
    \end{tabular}
    \caption{\sl \textbf{Segmentation results.} 
    The two models, $M_\text{DSA}$ and $M_\text{baseline}$ were trained with and without the DSA layer, respectively. Note that $M_\text{DSA}$ generates better segmentation masks with fewer holes and enforced connectivity. Also note the cleaner edges compared to the standard segmentation results. Additional figures are in the Appendix.}
    \label{fig:segout}
\end{figure}


We observe that our DSA layer provides more consistent segmentation maps and higher Jaccard scores than the baseline model; see \Figref{fig:segout}. For the Weizmann horse dataset, $M_\text{conn}$ enforces the connectivity of the segmented objects while also limiting noise in the segmentation map. In the cell segmentation task, we note that the number of segments is high while the objects are small. Since the size of the components is small, our DSA layer Jacobian exhibits substantial differences from the commensurate identity gradient for the baseline models. \Tabref{tab:seg_acc} also shows the further improvement in Jaccard score on cell segmentation tasks over the Weizmann horse dataset. 


\begin{table}[t!]
    \centering
    \caption{\sl \textbf{Results for the horse and cell segmentation dataset:} Jaccard scores for the baseline and connected component models for the cell and horse segmentation task. From independent three runs with random seeds and the table reports mean and standard deviation. As the objects of interest (piecewise constant components) are smaller, the model with the DSA layer learns a better representation. Predictions are thresholded at $0.5$.}
    \begin{tabular}{c c c}
    \toprule[1.5pt] 
    \multicolumn{1}{l}{\textbf{Dataset}} & Baseline ($M_\text{baseline}$) & Baseline + DSA ($M_\text{DSA}$)  \\
    \midrule 
    \multicolumn{1}{l}{Weizmann Horse~\citep{borenstein2004learning}} & 72.06 $\pm$ 0.60   & \textbf{73.13 $\pm$ 0.31} \\ 
    \multicolumn{1}{l}{Broad Bioimage Benchmark~\citep{ljosa2012annotated}} & 79.34 $\pm$ 0.43   & \textbf{81.56 $\pm$ 0.24}\\ 
    \bottomrule
    \end{tabular}
    \label{tab:seg_acc}
\end{table}

\noindent\textbf{3D point cloud reconstruction using NURBS:}
Next, we provide results for two experiments using DSA with NURBS discussed in \Secref{SubSec:NURBS}. The first application is surface fitting for a complex benchmark surface represented by a mesh of surface points obtained by evaluating the benchmark test function at these points. We use Bukin function N.6 (publicly available \href{https://www.sfu.ca/~ssurjano/bukin6.html}{\textcolor{blue}{here}}) for generating a grid of $256\times256$ points as shown on the left of \Figref{fig:nurbs_fitting}. For fitting a NURBS surface from the defined target point cloud, we initialize a uniform clamped knot vector for a cubic basis function and random control points of size $8\times8$ points. Using DSA, we evaluate the NURBS surface for a uniform grid of $256\times256$ parametric points. We now evaluate the surface and use mean squared error for fitting the surface point cloud using NURBS. We consider two scenarios: (i) we do not update the knot vectors (i.e., no reparameterization), and (ii) we compute the gradients for the knot vectors and allow for reparameterization (i.e., change of knot locations). We provide the comparison of these scenarios in \Tabref{Tab:nurbsfitting}. We see that the reparameterization helps in reducing the error in fit by half. Also, we notice that the density of points evaluated has a very minimal impact on the performance (see more details in Appendix). Visually, in \Figref{fig:nurbs_fitting}, we see that two knots in the ``$v$'' direction come close to each other around $0.06$, enabling a sharp edge in the evaluated surface. 

The next experiment we present involves surface reconstruction from point clouds using a graph convolutional neural network and DSA for unsupervised training. We use the SplineNet method proposed by \citet{sharma2020parsenet} to be the baseline for point cloud reconstruction using splines. SplineNet uses a dynamic graph convolutional neural network (DGCNN) to predict the control points for a spline surface. The authors use a supervised control point loss to perform the training and also include regularizations such as the Laplacian loss and a patch distance (using Chamfer distance) loss. Instead, we perform this training in an unsupervised manner by not using the control points prediction loss and only using DSA to evaluate the surface and then apply regularization of minimizing the Laplacian of the surface. Since we can train this in an unsupervised manner, we can even use an arbitrary number of control points and are not restricted to the target control points.

\begin{figure}[t!]
    \centering
    \includegraphics[width=0.99\linewidth, trim={2.0in 2.0in 2.0in 2.0in}, clip]{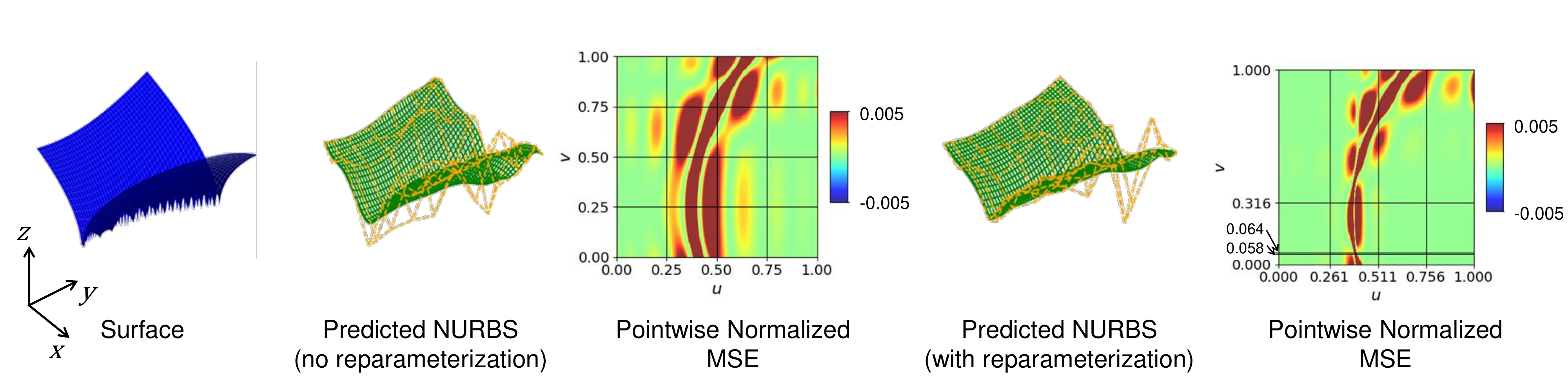}
    \caption{\sl \textbf{NURBS surface fitting results:} Surface fitting to point cloud generated using the Bukin's function N.6 given by $z=100\sqrt{|y-0.01x^2|}+0.01|x+10|; -15<x<-5, -3<y<3$. The center image shows the surface fit obtained without reparameterization of the knots. We obtain better fit by reparameterizing the knots.}
    \label{fig:nurbs_fitting}
\end{figure}

\begin{table*}[t!]
\centering
\caption{\sl \textbf{NURBS surface fitting results:} Comparison of mean squared error between the target surface point cloud and the surface generated using DSA with and without reparameterization.}
\label{Tab:nurbsfitting}
\setlength\extrarowheight{0.02in}
\newcommand{\tabincell}[2]{\begin{tabular}{@{}#1@{}}#2\end{tabular}}
    \begin{tabular}{c c c}
    \toprule[1.5pt]
 Number of Points & \tabincell{c}{$M_{DSA}$(without reparameterization)} & \tabincell{c}{$M_{DSA}$(with reparameterization)}  \\\hline

    $128\times128$ & 19.83 $\pm$ 0.001  & \textbf{8.25 $\pm$ 0.01} \\ 
    $256\times256$ & 19.85 $\pm$ 0.001  & \textbf{8.23 $\pm$ 0.02} \\ \hline

    \end{tabular}
\end{table*}

For a fair comparison, we use the same network, dataset, and hyperparameters as \citet{sharma2020parsenet} and change the loss functions by removing the control point regression loss. For comparison, we compute the chamfer distance between the input point cloud and the NURBS surface fit by the DGCNN model ($M_{DSA}$) (see Appendix for details of training). We use the Spline Dataset, which is a subset of surfaces extracted from the ABC dataset (available for public use under \href{https://deep-geometry.github.io/abc-dataset/#license}{this} license). In \Tabref{Tab:PointCloudReconstruction}, we provide a comparison of chamfer distance obtained between the predicted surface points from splines and the input point cloud for the test dataset. In our experiments, we observe that we get significantly better performance with fewer control points. This is because most of the surfaces in the dataset are simply curved surfaces that can be easily fit with fewer control points. 

\begin{table*}[t!]
\centering
\setlength\extrarowheight{0.02in}
\caption{\sl \textbf{Point-cloud reconstruction results:} Comparison between the model proposed by \citet{sharma2020parsenet} and its extension using DSA (with different number of control points). We compare the two sided chamfer distance (scaled by 100) between the input point cloud and the fitted surface.}
\label{Tab:PointCloudReconstruction}
\newcommand{\tabincell}[2]{\begin{tabular}{@{}#1@{}}#2\end{tabular}}
    \begin{tabular}{c c c c c}
    \toprule[1.5pt]
    \textbf{Experiment} & \tabincell{c}{$M_{baseline}$\\ ($20\times20$)} & \tabincell{c}{$M_{DSA}$\\ ($20\times20$)} & \tabincell{c}{$M_{DSA}$\\ ($5\times5$)} &  \tabincell{c}{$M_{DSA}$\\ ($4\times4$)}\\\hline

    \textbf{Chamfer Distance}& 1.18 $\pm$ 0.10 & 0.03 $\pm$ 0.02 & 0.14 $\pm$ 0.07 & \textbf{0.02 $\pm$ 0.01} \\ \hline

    \end{tabular}
    \vspace{0.03in}
\end{table*}


\noindent\textbf{PDE based surrogate physics priors:}\label{subsec:pde}
Finally, we leverage DSA in the context of solving PDEs as a prior. In particular, we consider the Poisson equation solved for $u$:
\begin{align}\label{eq:poisson-manufactured-pde}
-\grad\cdot({\nu}(\vec{x})\grad u) &= f(\vec{x}) \text{  in  } D \\
u|_{\partial D} &= 0\label{eq:poisson-manufactured-bc}
\end{align}
where $ D = [0,1]^2 $, a 2D square domain, $ \nu $ is the \textit{diffusivity} and $f$ is the forcing function. We consider two experiments here: (1) validation of our approach with an analytically known solution, and (2) extending this to learn the solutions for the parametric Poisson equation parameterized using $\nu$. 

For the first experiment, we set $\nu$ to $ 1 $ and the forcing $f = f(\mvec{x}) = f(x,y) = 2\pi^2\sin(\pi x) \sin(\pi y) $, and minimize the residual using the approach described in \Secref{sec:fem}. We know that for this PDE and the conditions provided, the exact solution is given by $ u_{ex}(x,y) = \sin(\pi x) \sin(\pi y) $. We compare our results ($u_{DSA}$) with the exact solution $u_{ex}$. Also, we perform this experiment with Lagrange polynomials of different degrees. Further, we compare our results with results obtained using PINNs~\citep{raissi2017physics1}. We obtain significantly better performance (lesser $\ell_2$-error by an order of magnitude) compared to PINNs, owing to more accurate gradients computed using our DSA approach. The performance improvement with increase in degree of polynomial in lower resolutions is more pronounced than at higher resolutions.

\begin{table*}[!t]
	\centering
	\caption{\sl \textbf{Quantitative comparison of Solving PDEs:} $L_2$ Norm between the analytical exact solution $u_{ex}$ and predicted $u$ using PINNs~\citep{raissi2017physics1} and DSA with different degrees of the Lagrange polynomials.
	}
    \resizebox{0.99\linewidth}{!}{
    \begin{tabular}{c c c c c c}
        \toprule[1.5pt]
        Model &  & $PINN$  & $DSA\,(d=1)$  & $DSA\,(d=2)$ & $DSA\,(d=3)$\\
        \midrule 
        \multirow{2}{*}{$L_2$ Norm} & $128\times 128$ &  3.72 $\pm$ 0.20 E-4  &  3.32 $\pm $ 0.05 E-5  & \textbf{2.16 $\pm$ 0.04 E-5}   & 2.37 $\pm$  0.10 E-5   \\
        & $256\times 256$ & 2.63 $\pm$ 0.20 E-4 & \textbf{2.57 $\pm$ 0.01 E-5} & 2.79 $\pm$ 0.20 E-5 & 2.59 $\pm$ 0.10 E-5 \\
        \bottomrule
        \end{tabular}
    }
    \label{tab:fe_comp}
\end{table*}

\begin{figure*}[!b]
	\centering
	\includegraphics[trim=4.0in 2.55in 4.0in 2.7in,clip,width=0.9\linewidth]{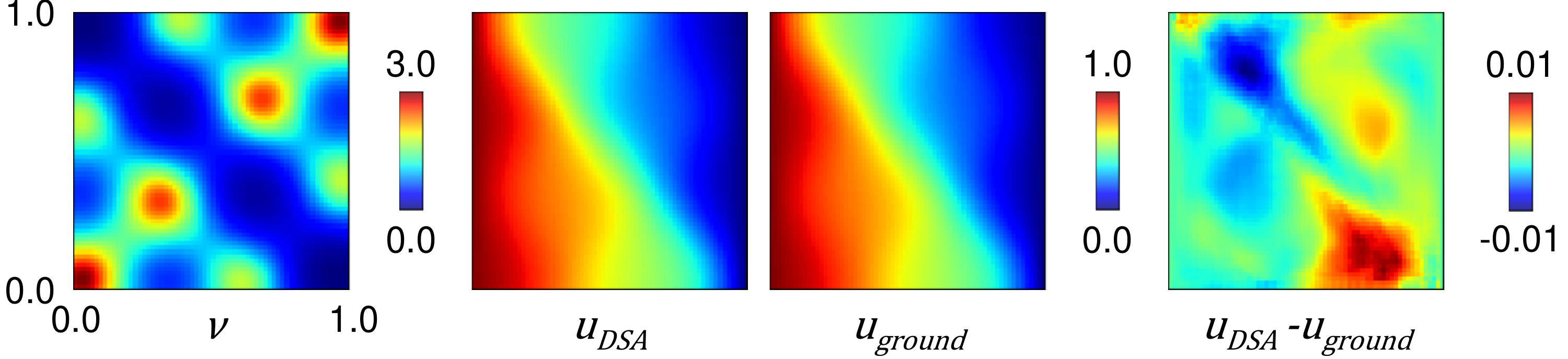}
	\caption{\sl \textbf{Learning a parametric family of PDE solutions:} Poisson's equation with log permeability coefficients  $\vec{\omega} = (-0.26, -0.77, -0.37, -0.92)$ in the Poisson's equation.
	}
	\label{fig:poisson-kl-contours-1}
\end{figure*}

Next, we present results for training a deep learning network with a prior for solving a \textit{parametric} Poisson's equation. The input to the network are different diffusivity maps $\nu$ sampled from 
\begin{align}\label{def:poisson-nu-harmonic-expansion}
	\nu(\vec{x}; \omega) 
	& = \exp\left(\sum_{i = 1}^{m} \omega_i \lambda_i \xi_i(x)\eta_i(y) \right)
\end{align}
where $\omega_i$ is an $m$-dimensional parameter, $ \lambda $ is a vector of real numbers with monotonically decreasing values arranged in order; and $ \xi $ and $ \eta $ are functions of $ x $ and $ y $ respectively. We take $ m = 4 $, $ \vec{\omega} = [-3,3]^4 $ and $ {\lambda_i} = \frac{1}{(1+0.25a_i^2)}$, where $ \vec{a} = (1.72, 4.05, 6.85, 9.82)$. Also $ \xi_i(x) = \frac{a_i}{2}\cos(a_i x) + \sin(a_i x) $ and $ \eta(y) = \frac{a_i}{2}\cos(a_i y) + \sin(a_i y) $. We generate several diffusivity maps by sampling this function with different values of $\vec{\omega}$. We use a UNet~\citep{ronneberger2015u} that takes these diffusivity maps and predicts the solution $u$, which is further optimized with the residual minimizing prior to the Poisson's equation. Thus, we obtain a trained neural network that predicts the solution field $u$ for any unknown diffusivity maps from the data distribution. We provide the predicted result along with its comparison with traditional numerical FEM results in \Figref{fig:poisson-kl-contours-1}. Visually, we see both the predicted solution field map ($u_{DSA}$) and the actual solution field ($u_{ground}$) obtained using traditional numerical methods match each other. The right most image shows the difference between both with the maximum deviation to be $0.01$, showing the accuracy of our (easy-to-implement) DSA-based FEM solver.


\section{Broader Impact and Discussion}
\label{sec:discussion}

We introduce a principled approach to estimate gradients for spline approximations. Specifically, we derive the (weak) Jacobian in the form of a block-sparse matrix based on the partitions generated by any spline approximation algorithm (which serves as the forward pass). The block structure allows for fast computation of the backward pass, thereby extending the application of differentiable programs (such as deep neural networks) to tasks involving splines. Our methods show superior performance than the state-of-the-art curve fitting methods by reducing the chamfer distance by an order of magnitude and the mean squared error in the case of surface fitting by a factor of two. Further, with the application of our methods in finite element analysis, we show significantly better performance compared to state-of-the-art physics-informed neural networks. 

Our method is quite generic and may impact applications such as computer graphics, physics simulations, and engineering design. Care should be taken to ensure that these applications are deployed responsibly. Future works include further algorithmic understanding of the inductive bias encoded by DSA layers, and dealing with splines having a dynamically chosen number of parameters (control points and knots).



\section*{Acknowledgements}
This work was supported in part by the National Science Foundation under grants CCF-2005804, LEAP-HI:2053760, CMMI:1644441, CPS-FRONTIER:1954556, USDA-NIFA:2021-67021-35329 and ARPA-E DIFFERENTIATE:DE-AR0001215. Any information provided and opinions expressed in this material are those of the author(s) and do not necessarily reflect the views of, nor any endorsements by, the funding agencies.


{
\bibliographystyle{unsrtnat}
\bibliography{NURBSML,ICLR}
}



\appendix
\section*{Appendix}
\section{Proofs and derivations}
\label{appdx:proof}
\setcounter{theorem}{0}
\begin{theorem}
    \label{thm:blockwise2}
    The Jacobian of the operation $F$ with respect to $x \in \R^n$ can be expressed as a \emph{block diagonal} matrix, $\rmJ \in \R^{n \times n}$, whose $(s,t)^{\textrm{th}}$ entry obeys:
    \begin{eqnarray}
        \rmJ_x(F(x)) (s, t) = \frac{\partial h(x)_s}{\partial x_t} = \begin{cases}
        \frac{\partial h_{I_i}(x)_s}{\partial x_t} &\textnormal{if } s, t \in I_i \\
        0 &\textnormal{otherwise}
    \end{cases}
    \end{eqnarray}
\end{theorem}

\begin{proof}
    The proof follows similar arguments as in Proposition 4 from \citet{Blondel2020FastDS}. 
    
    Let $\mathcal{I} = \{I_1, I_2,\cdots, I_k\}$ be $k$ partitions induced by some $H:\R \to \R$ for some input, $\rvx \in \R^n$ and $h \in \R^n$ be a vector from $n$ equally spaced evaluated $H$ in its domain. Then, each element, $x_i$ uniquely belongs to some partition $I_r$. 
    
    Now, 
    \begin{align*}
        \rmJ_x(F(x)) (s, t) &= \frac{\partial \sum_{j=1}^k h(x)_s \odot \mathbbm{1}(s \in I_j)}{\partial x_t} \\
        &= \begin{cases}
            \frac{\partial h(x)_s}{\partial x_t} & \text{if}~~s, t \in I_r \\
            0 & \text{otherwise}
        \end{cases}
    \end{align*}
    
    Note that this is a block-diagonal matrix with each block being $|I_r| \times |I_r|$, giving us the required statement.
\end{proof}

    
    
    

\section{Application of DSA to piecewise polynomial regression}
\label{appx:piecewise-poly}
\paragraph{1D piecewise constant regression:}
\label{subsec:const}

We first provide the notations we provided in Section~\ref{sec:diffpiece}. 

Let $f \in \R^n$ be a vector where the $i{\textsuperscript{th}}$ element is denoted as $f_i$. 
Let us use $[n] = \{1,2,\ldots,n\}$ to denote the set of all coordinate indices. For a vector $f \in \R^n$ and an index set $I \subseteq [n]$, let $f_I$ be the restriction of $f$ to $I$, i.e., for $i\in I$, we have $f_I(i):=f_i$, and $f_I(i):=0$ for $i \notin I$. Now, consider any fixed partition of $[n]$ into a set of disjoint intervals $\mathcal{I}=\{I_1, \ldots, I_k\}$ where the number of intervals $|\mathcal{I}|=k$. The $\ell_2$-norm of $f$ is written as $\|f\|_2:=\sqrt{\sum_{i=1}^n f_i^2}$ while the $\ell_2$ distance between $f, g$ is written as $\|f-g\|_2$. Finally, $\mathbbm{1}_I \in \{0, 1\}^n$ is a indicator vector where for $i \in I$, $\mathbbm{1}_I(i) = 1$ and for $i \notin I$, $\mathbbm{1}_I(i) = 0$.

We consider the case of $k$-piecewise regression in 1D, where we can use any algorithm to approximate a given input vector with a fixed number of piecewise polynomial functions. The simplest example is that of $k$-piecewise \emph{constant} regression, where a given input vector is approximated by a set of constant segments. 

Formally, consider a piecewise constant function $H: \R \to \R$ with $k$ pieces. Similar to spline, we evaluate $H$ at any $n$ equally spaced points in its domain. This gives us a vector $h \in \R^n$, which we call a $k$-piecewise constant vector. Since the best (in terms of $\ell_2$-norm) constant approximation to a function is its mean, a $k$-piecewise constant function approximation can be reparameterized over the collection of all disjoint intervals $\mathcal{I} = \{I_1, \ldots, I_k\}$ of $[n]$ such that given $\rvx$: 
\begin{align}
    &\min_{I_1,\ldots,I_k} \sum_{i=1}^n \sum_{j=1}^k (h_{I_j}(i) - x_i)^2 =\min_{I_1, \ldots, I_k} \sum_{j=1}^{k} \sum_{i \in I_j}( \frac{1}{|I_j|}\sum_{l\in I_j}x_l - x_i)^2 
\end{align}

We assume an optimal $H$ (parameterized by $\{I_i\}$ that can be obtained using many existing methods (a classical approach by dynamic programming~\citep{DBLP:conf/vldb/JagadishKMPSS98}). The running time of such approaches is typically $O(nk)$, which is constant for fixed $k$; see \citet{acharya2015fast} for a more detailed treatment.

Using Theorem~\ref{thm:blockwise}, the Jacobian of the output $k$-histogram with respect to $\rvx$ assumes the following form: 
\begin{align}
    \frac{\partial h}{\partial x_i} = \frac{\partial}{\partial x_i} \sum_{j=1}^k (\frac{1}{|I_j|}\sum_{l \in I_j} x_l)) &= \frac{\partial}{\partial x_i} \sum_{j=1}^k (\frac{1}{|I_j|}(\sum_{l \in I_j} x_l) \mathbbm{1}_{I_j}) \\
    &= \sum_{j=1}^k \frac{\partial}{\partial x_i} \frac{1}{|I_j|} (\sum_{l \in I_j} \mathbbm{1}_{I_j}) = \frac{1}{|I_j|} \mathbbm{1}_{I_j}
\end{align}

Therefore, the Jacobian of $h$ with respect to $\rvx$ forms the block-diagonal matrix $\mathbf{J} \in \mathbb{R}^{n \times n}$:
\begin{align*}
    \mathbf{J} =
    \begin{bmatrix}
        \mathbf{J}_1 & \mathbf{0} & \ldots & \mathbf{0} \\
        \mathbf{0} & \mathbf{J}_2 & \ldots & \mathbf{0} \\
        \vdots & \vdots & \ddots & \vdots \\
        \mathbf{0} & \mathbf{0} & \ldots & \rmJ_k
    \end{bmatrix} 
\end{align*}
where all entries of $\rmJ_i \in \mathbb{R}^{|I_i|\times|I_i|}$ equal to $1/|I_i|$. Note here that the sparse structure of the Jacobian allows for fast computation, and it can be easily seen that computing the Jacobian vector product $\rmJ^T \nu$ for any input $\nu$ requires $O(n)$ running time. As an additional benefit, the decoupling induced by the partition enables further speed up in computation via parallelization.

\paragraph{Generalization to 1D piecewise polynomial fitting:}
\label{subsec:genpoly}
We now derive differentiable forms of generalized piecewise $d$-polynomial regression, which is used in applications such as spline fittings.

As before, $H: \R \to \R$ is any algorithm to compute the $k$-piecewise $d$ polynomial approximation of an input vector $\rvx \in \R^d$ that outputs partition $\mathcal{I} = \{I_1, \ldots, I_k\}$. Similarly, the function $H$ gives us a vector $h \in \R^n$, a $k$-piecewise polynomial vector. Then, for each partition, we are required to solve a $d$-degree polynomial regressions. Generally, the polynomial regression problem is simplified to linear regression by leveraging a Vandermonde matrix. We get a similar closed-form expression for the coefficient as in Section~\ref{SubSec:NURBS}.

Assume that for partition $I_j$, the input indices $t_{I_j}(i)$ is i\textsuperscript{th} element in an index vector corresponding to the $I_j$ partition. Then, the input indices $t_{I_j}(i)$ are represented as a Vandermonde matrix, $\rmV_{I_j}$:
\begin{align*}
\mathbf{V}_{I_j} = 
\begin{bmatrix}
    1 & t_{I_j}(1) & t_{I_j}(1)^2 & \cdots & t_{I_j}(1)^d \\
    1 & t_{I_j}(2) & t_{I_j}(2)^2 & \cdots & t_{I_j}(2)^d \\
    \vdots & \vdots & \vdots & \ddots & \vdots \\
    1 & t_{I_j}(|I_j|) & t_{I_j}(|I_j|)^2 & \cdots & t_{I_j}(|I_j|)^d
\end{bmatrix} \, .
\end{align*}
It can be shown that the optimal polynomial coefficient $\alpha_{I_j}$ corresponding to the partition (or disjoint interval) $I_j$ have the following closed form:
$$\alpha_{I_j} = (\rmV_{I_j}^{T}\rmV_{I_j})^{-1}\rmV_{I_j}^T\rvx_{I_j},$$
where $\rvx_{I_j} \in \R^{|I_j|}$ is a vector $\rvx$ length of $|I_j|$ corresponding to the $I_j$ partition such that $\rvx_{I_j}(i) = x_i$ if $i \in I_j$ and undefined if $i \notin I_j$.
This can be computed in $O(knd^w)$ time where $w$ is the matrix-multiplication exponent~\citep{GKS06}. Then using Theorem~\ref{thm:blockwise} and the gradient for polynomial regression, the Jacobian of $h_{I_j}$ with respect to $\rvx$ forms a blockwise sparse matrix:
\begin{align*}
    \frac{\partial h_{I_j}(s)}{\partial x_l} &= \frac{\partial}{\partial x_l} (\langle \alpha_{I_j}, [\rmV_{I_j}^T]_s \rangle) = \frac{\partial}{\partial x_l} (\langle  (\rmV_{I_j}^{T}\rmV_{I_j})^{-1}\rmV_{I_j}^T\rvx_{I_j}, [\rmV_{I_j}^T]_s \rangle) \\
    &= \frac{\partial}{\partial x_l} [\rmV_{I_{j}}^T]_s^T (\rmV_{I_j}^T \rmV_{I_j})^{-1} \rmV_{I_j}^T \rvx_{I_j}\\ 
    &=
    \begin{cases}
        \big[\rmV_{I_j}(\rmV_{I_j}^T\rmV_{I_j})^{-1}[\rmV_{I_j}^T])_s\big]_l & \text{if } l, s \in I_j \\
        0 &\text{otherwise.}
    \end{cases}
\end{align*}

The two main takeaways here are as follows: (1) $\rmV_{I_i}$ can be precomputed for all possible $n-1$ partition sizes, thus allowing for fast ($O(n)$) computation of Jacobian-vector products; and (2) an added flexibility is that we can independently control the degree of the polynomial used in each of the partitions. The second advantage could be very useful for heterogeneous data as well as considering boundary cases in data streams.

\subsection{2D piecewise constant functions}
\label{appdx:2dpiecewise}
Our 1D piecewise spline approximation can be (heuristically) extended to 2D data. We provide the detailed descriptions. We consider the problem of image segmentation, which can be viewed as representing the domain of an image into a disjoint union of subsets. Neural-network based segmentation involves training a model (deep or otherwise) to map the input image to a {segmentation map}, which is a piecewise constant spline function. However, standard neural models trained in a supervised manner with image-segmentation map pairs would generate pixel-wise predictions, which could lead to disconnected regions (or holes) as predictions. 
We leverage our approach to enforce deep models to predict piecewise constant segmentation maps.  In case of 2D images, note that we do not have a standard primitive (for piecewise constant fitting) to serve as the forward pass. Instead, we leverage connected-component algorithms (such as Hoshen-Kopelman, or other, techniques~\citep{wu2005optimizing}) to produce a partition, and the predicted output is a piecewise constant image with values representing the mean of input pixels in the corresponding piece. For the backward pass, we use a tensor generalization of the block Jacobian where each partition is now represented as a channel which is only non-zero in the positions corresponding to the channel. Formally, if the image $\rvx \in \R^{n}$ is represented as the union of $k$ partitions, $h = \bigcup_{i=1}^k I_i$, the Jacobian, $\rmJ_{\rvx} = \partial h / \partial \rvx \in \R^{n\times n}$ and, 
\begin{equation}
    \rmJ_{\rvx}(F(x))(s, t) = \begin{cases} \frac{\partial h(x)_s}{\partial x_t} = \frac{1}{|I_i|} & \text{if}~~s, t \in I_i,  \\
                                0 & \text{otherwise.}
                    \end{cases} 
\end{equation}
Note that $I_i$ here no longer correspond to single blocks in the Jacobian. Here, they will reflect the positions of pixels associated with the various components. However, the Jacobian is still sparsely structured, enabling fast vector operations.

\section{Implementing DSA with NURBS}

\subsection{Backward evaluation for NURBS surface}\label{SubSec:Backward}
In a modular machine learning system, each computational layer requires the gradient of a loss function with respect to the output tensor for the backward computation or the backpropagation. For our NURBS evaluation layer this corresponds to $\nicefrac{\partial{\mathcal{L}}}{\partial{\mathcal{S}}}$ . As an output to the backward pass, we need to provide $\nicefrac{\partial{\mathcal{L}}}{\partial{\Psi}}$. While we represent $\mathcal{S}$ for the boundary surface, computationally, we only compute $\vec{S}$ (the set of surface points evaluated from $\mathcal{S}$). Therefore, we would be using the notation of ${\partial{\vec{S}}}$ instead of ${\partial{\mathcal{S}}}$ to represent the gradients with respect to the boundary surface. Here, we make an assumption that with increasing the number of evaluated points, ${\partial{\vec{S}}}$ will asymptotically converge to ${\partial{\mathcal{S}}}$. Now, we explain the computation of $\nicefrac{\partial{\vec{S}}}{\partial{\Psi}}$ in order to compute $\nicefrac{\partial{\mathcal{L}}}{\partial{\Psi}}$ using the chain rule. In order to explain the implementation of the backward algorithm, we first explain the NURBS derivatives for a given surface point with respect to the different NURBS parameters.

\subsection {NURBS derivatives}
We rewrite the NURBS formulation as follows:
\begin{equation}
\vec{S}(u,v) = \frac{\vec{NR}(u,v)}{w(u,v)}
\label{eq:NURBSsplit}
\end{equation}
where,  
\begin{equation*}
\vec{NR}(u,v) = \sum_{i=0}^n{ \sum_{j=0}^m{N_i^p(u) N_j^q(v)w_{ij}\vec{P}_{ij} } }
\end{equation*}

\begin{equation*}
{w}(u,v) = \sum_{i=0}^n\sum_{j=0}^m{N_i^p(u) N_j^q(v)w_{ij}}
\end{equation*}

For the forward evaluation of $\vec{S}(u,v) =\mathbf{f} \left( \vec{P}\,,\vec{U}\,,\vec{V}\,,\vec{W} \right)$, we can define four derivatives for a given surface evaluation point: $\vec{S}_{,u} :=  \nicefrac{\partial{\vec{S}(u,v)}}{\partial{u}}$, $\vec{S}_{,v} :=  \nicefrac{\partial{\vec{S}(u,v)}}{\partial{v}}$, $\vec{S}_{,\vec{P}} :=  \nicefrac{\partial{\vec{S}(u,v)}}{\partial{\vec{P}}}$, and $\vec{S}_{,\vec{W}} :=  \nicefrac{\partial{\vec{S}(u,v)}}{\partial{\vec{W}}}$. Note that,  $\vec{S}_{,\vec{P}}$ and $\vec{S}_{,\vec{W}}$ are represented as a vector of gradients $\{\vec{S}_{,P_{ij}} \forall P_{ij} \in \vec{P}\}$ and $\{\vec{S}_{w_{ij}} \forall w_{ij} \in \vec{W}\}$. Now, we show the mathematical form of each of these four derivatives. The first derivative is traditionally known as the parametric surface derivative, $\vec{S}_{,u}$. Here, $N_{i,u}^p(u)$ refers to the derivative of basis functions with respect to $u$.

\begin{equation}
	\vec{S}_{,u}(u,v)= \frac{\vec{NR}_{,u}(u,v)w(u,v) - \vec{NR}(u,v)w_{,u}(u,v)}{w(u,v)^2}
	\label{eq:NURBSuDerivative1}
\end{equation}

where,

\begin{equation*}
	\vec{NR}_{,u}(u,v)=\sum_{i=0}^n{\sum_{j=0}^m{N_{i,u}^p(u) N_j^q(v)w_{ij}\vec{P}_{ij} }}
\end{equation*}

\begin{equation*}
{w}_{,u}(u,v) = \sum_{i=0}^n\sum_{j=0}^m{N_{i,u}^p(u) N_j^q(v)w_{ij}}
\end{equation*}

A similar surface point derivative could be defined for $\vec{S}_{,v}$. These derivatives are useful in the sense of differential geometry of NURBS for several CAD applications~\citep{Krishnamurthy-2009}. However, since many deep learning applications such as surface fitting are not dependent on the $(u,v)$ parametric coordinates, we do not use it in our layer. Also, note that $\vec{S}_{,u}$  and $\vec{S}_{,v}$ are not the same as $\vec{S}_{,\vec{U}}$ and $\vec{S}_{,\vec{V}}$. A discussion about $\vec{S}_{,\vec{U}}$ and $\vec{S}_{,\vec{V}}$ is provided later in this section. Now, let us define $\vec{S}_{,p_{ij}}(u,v)$.

\begin{equation}
	\vec{S}_{,\vec{P}_{ij}}(u,v)= \frac{N_i^p(u) N_j^q(v)w_{ij}}{\sum_{k=0}^n\sum_{l=0}^m{N_k^p(u) N_l^q(v)w_{kl}}}
	\label{eq:NURBSuDerivative4}
\end{equation}

$\vec{S}_{,\vec{P}_{ij}}(u,v)$ is the rational basis functions themselves. Computing $\vec{S}_{,w_{ij}}(u,v)$ is more involved with $w_{ij}$ terms in both the numerator and the denominator of the evaluation.

\begin{equation}
	\vec{S}_{,w_{ij}}(u,v)=\frac{\vec{NR}_{,w_{ij}}(u,v)w(u,v) - \vec{NR}(u,v)w_{,w_{ij}}(u,v)}{w(u,v)^2}
	\label{eq:NURBSvDerivative5}
\end{equation}
where, 
\begin{equation*}
	\vec{NR}_{,w_{ij}}(u,v)={N_{i}^p(u) N_j^q(v)\vec{P}_{ij}}
\end{equation*}

\begin{equation*}
    {w}_{,w_{ij}}(u,v) = {N_{i}^p(u) N_j^q(v)}
\end{equation*}

\subsection{Derivatives with respect to knot points}
For simplicity, we will stick to 1D NURBS curves. The extension to 2D surfaces is straightforward by taking Kronecker products.

We recall the definition of the NURBS basis:
\begin{equation}
	N_i^d(u)=\frac{u-u_i}{u_{i+d}-u_i}N_i^{d-1}(u)+\frac{u_{i+d+1}-u}{u_{i+d+1}-u_{i+1}}N_{i+1}^{d-1}(u),~
	N_i^0(u) = \left\{
	\begin{array}{l l}
		1  \quad \mbox{if $u_i\leq u\leq u_{i+1}$} \\
		0  \quad \mbox{otherwise}
	\end{array} \right.
\end{equation}

The goal is to evaluate the derivative of $N_i^d(u)$ with respect to the knot points $\{u_i\}$. We observe that due to the recursive nature of the definition, we can accordingly compute the derivatives of $N_i^d(u)$ in a recursive fashion using chain rule, \emph{provided} we can evaluate:
\[
\frac{\partial N_i^0(u)}{\partial u_i} = \frac{\partial \mathbf{1}([u_i,u_{i+1}])}{\partial u_i}
\]
(and likewise for $u_{i+1}$) where $\mathbf{1}$ denotes the indicator function over an interval. However, this derivative is not well-defined since the gradient is zero everywhere and undefined at the interval edges.

We propose to approximate this derivative using \emph{Gaussian smoothing}. Rewrite the interval as the difference between step functions convolved with deltas shifted by $u_i$ and $u_{i+1}$ respectively:
\[
\mathbf{1}([u_i,u_{i+1}))(u) = \text{sign}(u) \star \delta(u - u_{i}) - \text{sign}(u) \star \delta(u - u_{i+1})
\]
and approximate the delta function with a Gaussian of sufficiently small (but constant) bandwidth:
\[
\mathbf{1}([u_i,u_{i+1}])(u) = \text{sign}(u) \star G_\sigma(u - u_{i}) - \text{sign}(u) \star G_\sigma(u - u_{i+1})
\]
where
\[
G_\sigma(u - \mu) = \frac{1}{\sqrt{2\pi \sigma^2}} \exp(- \frac{(u-\mu)^2}{2\sigma^2}).
\]
The derivative with respect to $\mu$ is therefore given by:
\[
G_\sigma'(u=\mu) = \frac{(u - \mu)}{2\sigma^2} G_\sigma(u - \mu) ,
\]
which means that the approximate gradient introduces a multiplicative $(u - \mu)$ factor with the original basis function. Propagating this through the chain rule and applying a similar strategy as Cox-de Boor recursion gives us Algorithm~\ref{Alg:Backward}. \qedsymbol

\section{Experimental details}
\label{appdx:expts}

\subsection{Segmentation}

\paragraph{Weizmann Horse dataset:} The dataset consists of 378 images of single horses with varied backgrounds and their corresponding ground truth. We divide the dataset into 85:15 ratios for training and testing, respectively. Further, each image is normalized to a $[0,1]$ domain by dividing it by $256$. 5443

\paragraph{Cell dataset:} The dataset consists of 19K gray-scale images containing various cells, and we take $1900$ subset images as the dataset. We divide the dataset into 85:15 ratios for training and testing, respectively. Similarly, we normalize the image to a $[0, 1]$ by dividing each pixel by $256$.

\paragraph{Architecture and training:} We use the following U-Net architecture for training our segmentation networks. While we use the equivalent model skeleton reported by \citet{ronneberger2015u}, we scale down the network size starting the initial channels $C=8$ (default channel is $C=64$). In both dataset, we train the network 1000 epochs with initial learning rate $0.0003$. We leverage Adam optimizer with $\beta = (0.9, 0.999)$ and weight decay $0.0001$. We use a binary cross entropy loss function as the objective function.

\subsection{NURBS surface fitting implementation}
The complete algorithm for forward evaluation of $\Vec{S}(u,v)$ as described in \citet{10.5555/265261} can be divided into three steps: 
\begin{enumerate}
\item Finding the knot span of $u \in [u_i,u_{i+1})$ and the knot span of $v \in [v_j,v_{j+1})$, where $u_i, u_{i+1} \in \mathbf{U}$ and $v_ j, v_{j+1} \in \mathbf{V}$. This is required for the efficient computation of only the non-zero basis functions.
\item Now, we compute the non-zero basis functions $N_i^p(u)$ and $N_j^q(v)$ using the knot span. The basis functions have specific mathematical properties that help us in evaluating them efficiently. The partition of unity and the recursion formula ensures that the basis functions are non-zero only over a finite span of $p+1$ control points. Therefore, we only compute those $p+1$ non-zero basis functions instead of the entire $n$ basis function. Similarly in the $v$ direction we only compute $q+1$ basis functions instead of $m$.
\item We first compute the weighted control points $\vec{P}^w_{ij}$ for a given control point $\vec{P}_{ij}=\{\vec{P}_x, \vec{P}_y , \vec{P}_z\}$ and weight $w_{ij}$ as $\{\vec{P}_x w, \vec{P}_y w, \vec{P}_z w\}$ representing the surface after homogeneous transformation for ease of computation. Once the basis functions are computed we multiply the non-zero basis functions with the corresponding weighted control points, $\vec{P}^w_{ij}$. This result, $\vec{S'}$ is then used to compute $\vec{S}(u,v)$ as $\{ S'_{x}/S'_{w}, S'_{y}/S'_{w}, S'_{z}/S'_{w}\}$.
\end{enumerate}

\begin{algorithm}[h!]
    \caption{Forward algorithm for multiple surfaces\label{Alg:Forward}}
    \SetKwInOut{Input}{Input}
    \SetKwInOut{Output}{Output}

    \Input{$\vec{U}$, $\vec{V}$, $\vec{P}$, $\vec{W}$, output resolution $n_{grid}$, $m_{grid}$}
    \Output{$\vec{S}$}
    
    \text{Initialize a meshgrid of parametric coordinates}\\
    \text{\hspace{0.3in}uniformly from $[0,1]$ using $n_{grid}\times m_{grid}$ : $u_{grid} \times v_{grid}$}\\
    \text{Initialize: $\vec{S} \rightarrow \vec{0}$}\\
    \For{$k = 1: surfaces$ in \textbf{\emph{parallel}}}
    {
    \For{$j=1:m_{grid}$ points in \textbf{\emph{parallel}}}
      {
        \For{$i=1:n_{grid}$ points in \textbf{\emph{parallel}}}
          {
            Compute $u_{span}$ and $v_{span}$ for the corresponding $u_i$ and $v_i$ using knot vectors $\vec{U_k}$ and $\vec{V_k}$ \\
            Compute basis functions $N_i$ and $N_j$ basis functions using $u_{span}$ and $v_{span}$ and knot vectors $\vec{U_k}$ and $\vec{V_k}$\\
            Compute surface point $\vec{S}(u_i,v_j)$ (in $x$, $y$, and $z$ directions).\\
            Store $u_{span}$, $v_{span}$, $N_i$, $N_j$, and $\vec{S}(u_i,v_j)$ for backward computation
          }
      }
  
    }
\end{algorithm}

In a deep learning system, each layer is considered as an independent unit performing the computation. The layer takes a batch of input during the forward pass and transforms them using the parameters of the layer. Further, in order to reduce the computations needed during the backward pass, we store extra information for computing the gradients during the forward computation. The NURBS layer takes as input the control points, weights, and knot vectors for a batch of NURBS surfaces. We define a parameter to control the number of points evaluated from the NURBS surface. We define a mesh grid of a uniformly spaced set of parametric coordinates $u_{grid}\times v_{grid}$. We perform a parallel evaluation of each surface point $S(u,v)$ in the $u_{grid}\times v_{grid}$ for all surfaces in the batch and store all the required information for the backward computation. The complete algorithm is explained in \Algref{Alg:Forward}. Our implementation is robust and modular for different applications. For example, if an end-user desires to use this for a B-spline evaluation, they need to set the knot vectors to be uniform and weights $\vec{W}$ to be $1.0$. In this case, the forward evaluation can be simplified to $\vec{S}(u,v) = \vec{f}(\vec{P})$. Further, we can also pre-compute the knot spans and basis functions during the initialization of the NURBS layer. During computation, we could make use of tensor comprehension that significantly increases the computational speed. We can also handle NUBS (Non-Uniform  B-splines), where the knot vectors are still non-uniform, but the weights $W$ are set to $1.0$. Note in the case of  B-splines $\Psi = \{\vec{P}\}$ (the output from the deep learning framework) and in the case of NUBS $\Psi = \{\vec{P}, \vec{U}, \vec{V}\}$.

\paragraph{SplineNet training details:}
The SplineNet architecture made of a series of dynamic graph convolution layers, followed by an adaptive max pooling and conv1d layers. We use the Chamfer distance as the loss function. The Chamfer distance ($\mathcal{L}_{CD}$) is a global distance metric between two sets of points as shown below.

\begin{equation}
    \mathcal{L}_{CD} = \sum_{\vec{P_i}\in\vec{P}}{\,\min_{\vec{Q_j}\in\vec{Q}}{{||\vec{P_i}-\vec{Q_j}}||_2}} + \sum_{\vec{Q_j}\in\vec{Q}}{\,\min_{\vec{P_i}\in\vec{P}}{||{\vec{P_i}-\vec{Q_j}}||_2}}
    \label{eq:ChamferDistance}
\end{equation}

For training and testing our experiments, we use the SplineDataset provided by \citet{sharma2020parsenet}. The SplineDataset is a diverse collection of open and closed splines that have been extracted from one million CAD geometries included in the ABC dataset. We run our experiments on open splines that are split into 3.2K, 3K, and 3K surfaces for training, testing, and validation. 

\subsection{PDE solver implementation with DSA prior}

Deep convolutional neural networks are a natural choice for the network architecture for solving PDEs due to the structured grid representation of $\mathcal{S}^d$ and similarly structured representation of $U^d_\theta$. The spatial localization of convolutional neural networks helps in learning the interaction between the discrete points locally. Since the network takes an input of a discrete grid representation (similar to an image, possibly with multiple channels) and predicts an output of the solution field of a discrete grid representation (similar to an image, possibly with multiple channels), this is considered to be similar to an image segmentation or image-to-image translation task in computer vision. U-Nets~\citep{ronneberger2015u} have been known to be effective for applications such as semantic segmentation and image reconstruction. Due to its success in diverse applications, we choose U-Net architecture for solving the PDE. The architecture of the network is shown in Figure~\ref{fig:unetarch}. First, a block of convolution and instance normalization is applied. Then, the output is saved for later use during skip-connection. This intermediate output is then downsampled to a lower resolution for a subsequent convolution block and instance normalization layers. This process is continued twice. The upsampling starts where the saved outputs of similar dimensions are concatenated with the output of upsampling for creating the skip-connections followed by a  convolution layer. LeakyReLU activation was used for all the intermediate layers. The final layer has a Sigmoid activation. 

 \begin{figure}[t!]
     \centering
     \includegraphics[width=0.8\linewidth]{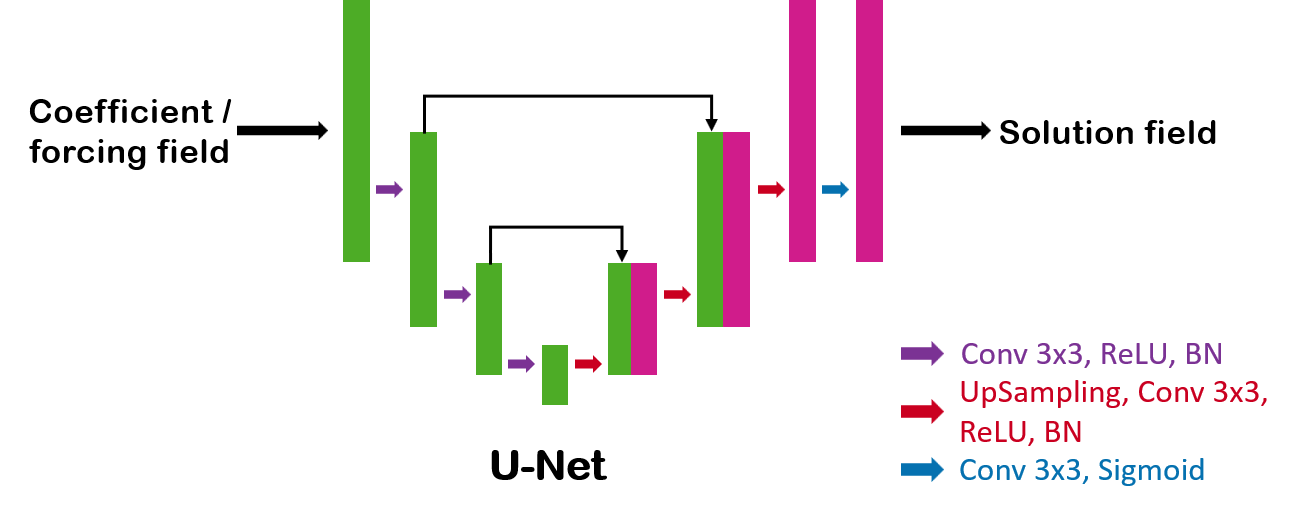}
     \caption{UNet architecture used for training}
     \label{fig:unetarch}
 \end{figure}

\subsubsection{Applying boundary conditions}\label{sec:boundary-conditions}
The Dirichlet boundary conditions are applied exactly. The query result from $ U^d_\theta $ from the network pertains only to the interior of the domain. The boundary conditions need to be taken into account separately. There are two ways of doing this:
\begin{itemize}[left=0in]
	\item Applying the boundary conditions exactly (this is possible only for Dirichlet conditions in FEM/FDM, and the zero-Neumann case in FEM)
	\item Taking the boundary conditions into account in the loss function, thereby applying them approximately.
\end{itemize}

\begin{figure*}[t!]
	\centering
	\includegraphics[trim=150 5 100 10,clip,width=0.9\linewidth]{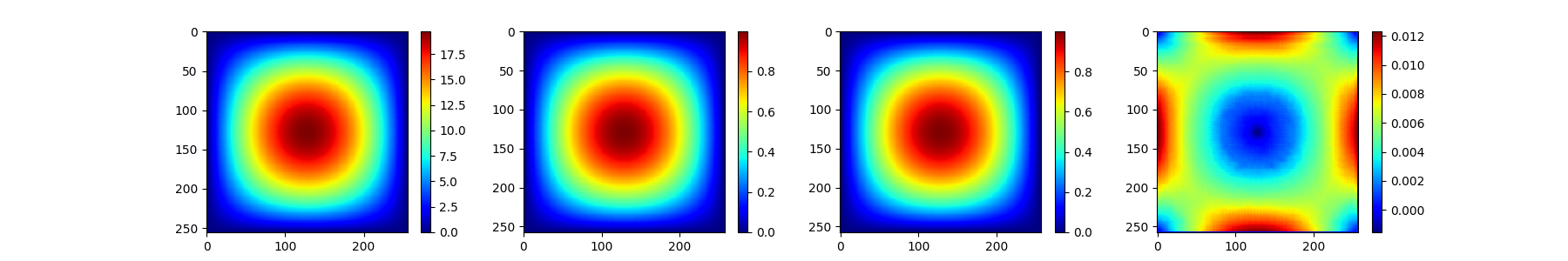}
	\caption{Solution to the linear Poisson's equation with forcing. From left to right: $f$, $u_{DSA}$, $u_{num}$ and ($u_{DSA}-u_{num}$). Here $u_{num}$ is a conventional numerical solution obtained through FEM. Diffusivity $\nu = 1$}
	\label{fig:poisson-manufactured-contours}
\end{figure*}

We take the first approach of applying the Dirichlet conditions exactly (subject to the mesh). Since the network architecture is well suited for 2d and 3d matrices (which serve as an adequate representation of the discrete field in 2D/3D on regular geometry), the imposition of Dirichlet boundary conditions amounts to simply padding the matrix by the appropriate values. A zero-Neumann condition can be imposed by taking the ``edge values" of the interior and copying them as padding. A nonzero Neumann condition is slightly more involved in the FDM case since additional equations need to be constructed, but if using FEM loss, this can be done with another surface integration on the relevant boundary.

\newpage
\section{Additional results}

\begin{figure}[h!]
    \def\lm{0.16\linewidth}
    \centering
    \setlength{\tabcolsep}{4pt}
    \begin{tabular}{c c c c}
        Original & Ground & $M_{\text{baseline}}$ & $M_{\text{DSA}}$ \\
        \frame{\includegraphics[width=\lm]{images/horse/unet_origin_0.png}} & 
        \frame{\includegraphics[width=\lm]{images/horse/unet_gnd_0.png}} & 
        \frame{\includegraphics[width=\lm]{images/horse/unet_base_0.png}} &
        \frame{\includegraphics[width=\lm]{images/horse/unet_conn_0.png}} \\
        \frame{\includegraphics[width=\lm]{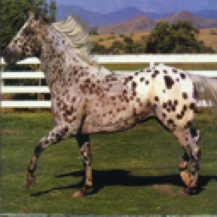}} & 
        \frame{\includegraphics[width=\lm]{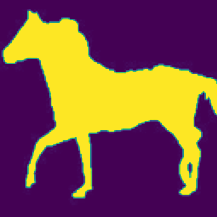}} & 
        \frame{\includegraphics[width=\lm]{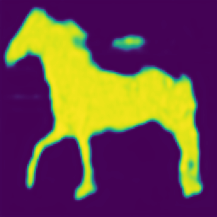}} &
        \frame{\includegraphics[width=\lm]{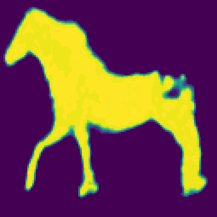}} \\
        \frame{\includegraphics[width=\lm]{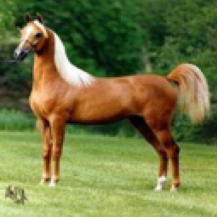}} & 
        \frame{\includegraphics[width=\lm]{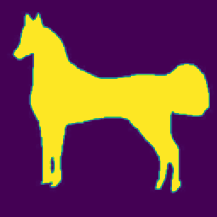}} & 
        \frame{\includegraphics[width=\lm]{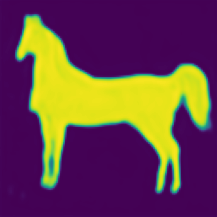}} &
        \frame{\includegraphics[width=\lm]{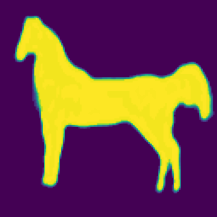}} \\
        \frame{\includegraphics[width=\lm]{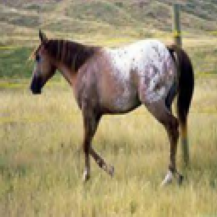}} & 
        \frame{\includegraphics[width=\lm]{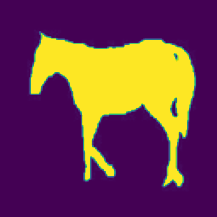}} & 
        \frame{\includegraphics[width=\lm]{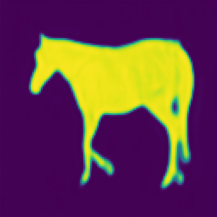}} &
        \frame{\includegraphics[width=\lm]{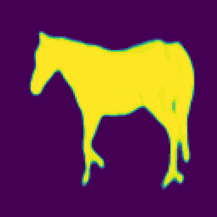}} \\
        \frame{\includegraphics[width=\lm]{images/horse/unet_origin_8.png}} & 
        \frame{\includegraphics[width=\lm]{images/horse/unet_gnd_8.png}} & 
        \frame{\includegraphics[width=\lm]{images/horse/unet_base_8.png}} &
        \frame{\includegraphics[width=\lm]{images/horse/unet_conn_8.png}} \\
        \frame{\includegraphics[width=\lm]{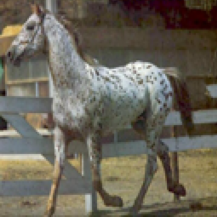}} & 
        \frame{\includegraphics[width=\lm]{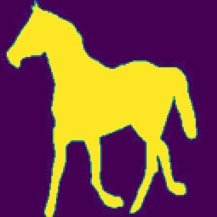}} & 
        \frame{\includegraphics[width=\lm]{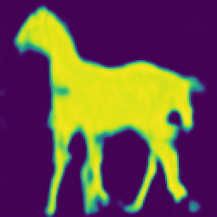}} &
        \frame{\includegraphics[width=\lm]{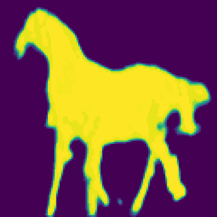}} \\
        \frame{\includegraphics[width=\lm]{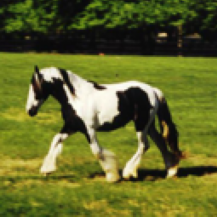}} & 
        \frame{\includegraphics[width=\lm]{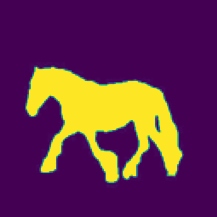}} & 
        \frame{\includegraphics[width=\lm]{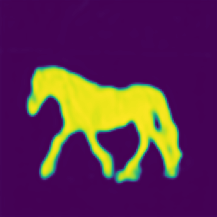}} &
        \frame{\includegraphics[width=\lm]{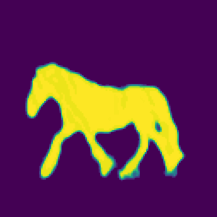}} \\
    \end{tabular}
    \caption{\textbf{Image Segmentation Tasks} Adding DSA layers ($M_{\text{DSA}}$) on top of U-Net ($M_{\text{baseline}}$) improves the segmentation tasks on both datasets.}
\end{figure}

\begin{figure}[h!]
    \centering
    \setlength{\tabcolsep}{1pt}
    \renewcommand{\arraystretch}{0.2}
    \def\sw{0.16\linewidth}
    \begin{tabular}{c c c c}
       Original & Ground & $M_\text{baseline}$ & $M_\text{DSA}$ \\
       \includegraphics[width=\sw]{images/cell_segmentation/img_11.png} & 
       \includegraphics[width=\sw]{images/cell_segmentation/cell_ground_11.png} &
       \includegraphics[width=\sw]{images/cell_segmentation/cell_noconn_pred_11.png} &
       \includegraphics[width=\sw]{images/cell_segmentation/cell_conn_pred_11.png} \\
            \includegraphics[width=\sw]{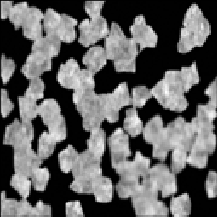} & 
       \includegraphics[width=\sw]{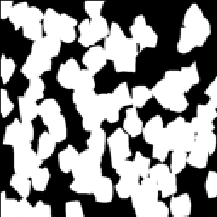} &
       \includegraphics[width=\sw]{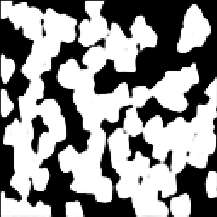} &
       \includegraphics[width=\sw]{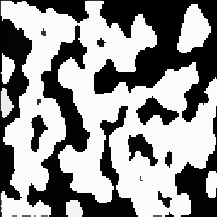} \\
            \includegraphics[width=\sw]{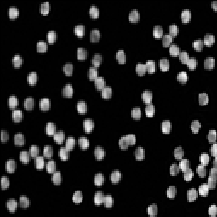} & 
       \includegraphics[width=\sw]{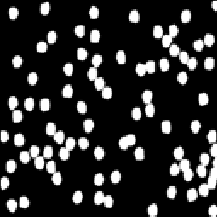} &
       \includegraphics[width=\sw]{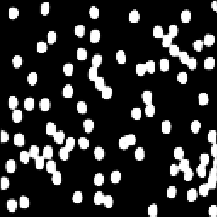} &
       \includegraphics[width=\sw]{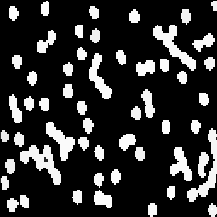} \\
            \includegraphics[width=\sw]{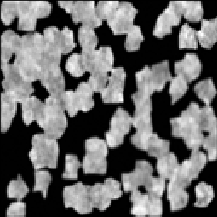} & 
       \includegraphics[width=\sw]{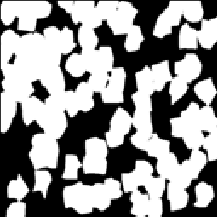} &
       \includegraphics[width=\sw]{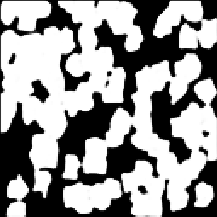} &
       \includegraphics[width=\sw]{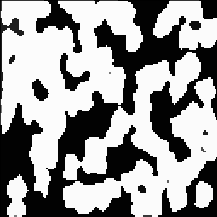} \\
            \includegraphics[width=\sw]{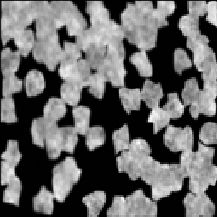} & 
       \includegraphics[width=\sw]{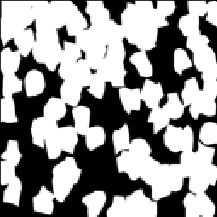} &
       \includegraphics[width=\sw]{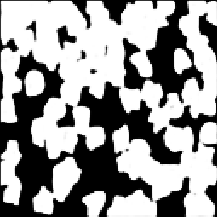} &
       \includegraphics[width=\sw]{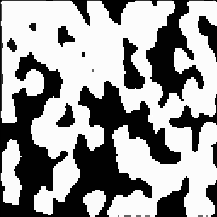} \\
            \includegraphics[width=\sw]{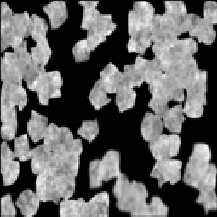} & 
       \includegraphics[width=\sw]{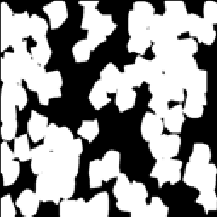} &
       \includegraphics[width=\sw]{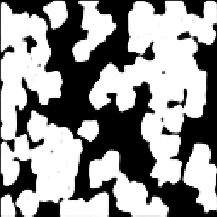} &
       \includegraphics[width=\sw]{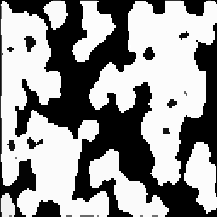} \\
            \includegraphics[width=\sw]{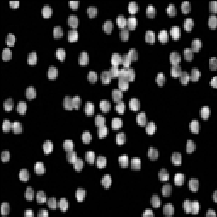} & 
       \includegraphics[width=\sw]{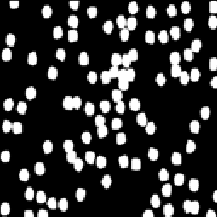} &
       \includegraphics[width=\sw]{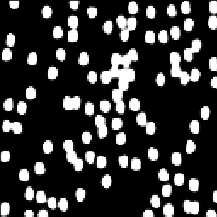} &
       \includegraphics[width=\sw]{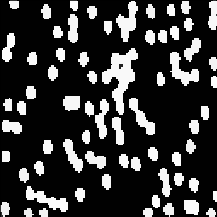} \\
            \includegraphics[width=\sw]{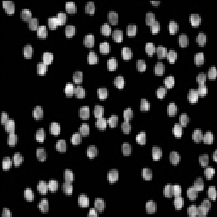} & 
       \includegraphics[width=\sw]{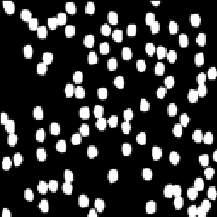} &
       \includegraphics[width=\sw]{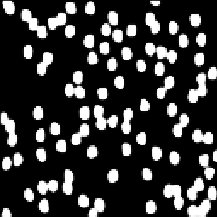} &
       \includegraphics[width=\sw]{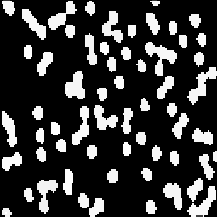} \\
    \end{tabular}

    \caption{Additional cell segmentations results. $M_{\text{baseline}}$ and $M_{\text{DSA}}$ correspond to U-Net and U-Net+DSA layers, respectively.}
    \label{fig:appdx_cell}
\end{figure}

\end{document}